\documentclass[11pt]{article}
\usepackage[figuresright]{rotating}
\usepackage{etex}
\usepackage{pstricks-add}
\reserveinserts{28}
\usepackage[hidelinks,pageanchor,dvips]{hyperref}  
\hypersetup{
  colorlinks   = true, 
  urlcolor     = blue, 
  linkcolor    = blue, 
  citecolor   = red 
}
\usepackage[pageref]{backref}
\usepackage{verbatim} 
\usepackage{hypernat}  
\usepackage{mathrsfs}
\usepackage{dsfont} 
\usepackage{amstext}
\usepackage{amsmath}
\usepackage{amsfonts}
\usepackage{amsthm,amsbsy,bm}
\usepackage{dsfont} 
\usepackage{subeqnarray}
\usepackage{empheq} 
\usepackage{color} 
\usepackage{amssymb}
\usepackage{exscale}  
\usepackage{stmaryrd}  
\usepackage{epsfig}
\usepackage{subfigure}
\usepackage{calc}
\usepackage{amssymb}
\usepackage{amstext}
\usepackage{amsmath}
\usepackage{multicol}
\usepackage{times}
\usepackage{fancyhdr}
\usepackage[small]{caption}
\usepackage{dsfont} 
\usepackage{subeqnarray}
\usepackage{empheq} 
\usepackage{color} 
\usepackage{verbatim}
\usepackage{url}
\usepackage{breakurl}
\usepackage{slashed}  
\usepackage{exscale}  
\usepackage{stmaryrd}  
\usepackage{subeqnarray}
\usepackage{cases}
\newfont{\boldit}{cmbxti10} 
\usepackage{fancyhdr} 
\usepackage{textcomp}
\newtheorem{theorem}{Theorem}[section]

\newtheorem{lemma}{Lemma}[section]
\newtheorem{proposition}{Proposition}[section]
\theoremstyle{definition}
\newtheorem{definition}{Definition}[section]

\theoremstyle{plain}
\theoremstyle{plain}
\swapnumbers \swapnumbers 
\usepackage{amsmath,varwidth,array,ragged2e,tabularx}
\usepackage{tikz}

\usepackage{graphicx} 
\usepackage{multirow,makecell,rotating} 
\newcolumntype{L}{>{\varwidth[c]{\linewidth}}l<{\endvarwidth}}
\newcolumntype{M}{>{$}l<{$}}
\newcommand{\tc}[1]{\multicolumn{1}{c}{#1}} 
\renewcommand*{\backref}[1]{} 
\renewcommand*{\backrefalt}[4]{%
    \ifcase #1 (Not cited.)%
    \or        (Cited on page~#2.)%
    \else      (Cited on pages~#2.)%
    \fi}
\allowdisplaybreaks[3]
\makeatletter
\def\hlinewd#1{%
\noalign{\ifnum0=`}\fi\hrule \@height #1 %
\futurelet\reserved@a\@xhline}

\makeatother

\makeatletter
\def\@cite#1#2{\textsuperscript{[{#1\if@tempswa , #2\fi}]}}
\makeatother

\numberwithin{equation}{section}
\begin{document}

\title{\Large A Unified Framework to Elementary Geometric Transformation Representation}
\author{ {}\thanks{correspondence author}\\
 }
\date{}
\maketitle
\begin{center}
{\bf Abstract}\\
\vspace{0.5cm}
\parbox{6.0in}
{\setlength
{\baselineskip}{0.5cm}
As an extension of projective homology, stereohomology is proposed via an extension of Desargues theorem and the extended Desargues configuration. Geometric transformations such as reflection, translation, central symmetry, central projection, parallel projection, shearing, central dilation, scaling, and so on are all included in stereohomology and represented as modified Householder elementary matrices. Hence all these geometric transformations are called {\em elementary}. This makes it possible to represent these elementary geometric transformations in homogeneous square matrices independent of a particular choice of coordinate system. 
} \end{center}
\vspace{0.5 cm}

\section{Introduction}

\noindent It is nearly 200 years since homogeneous coordinates were introduced by Pl\"{u}ker and M\"{o}bius \cite[pp.852~\texttildelow~854]{Kline} individually as part of the algebraic scheme for projective geometry. It is only by homogeneous representation that some basic geometric transformations can be represented into square matrices.  Such geometric transformations are translation, perspective projection, rotation with axis not passing through the coordinate system origin, shearing, and so on. The homogeneous representation of geometric transformations in square matrices based on homogeneous coordinates is powerful and elegant which has now been widely adopted in both projective geometry and computer graphics.

However, the current homogeneous representation of geometric transformations still exhibits deficiency. A fundamental issue is that conventional homogeneous representation has not yet presented a {\em nice definition} to the homogeneous geometric transformations.  The difficulty of algebraically defining homogeneous geometric transformations in projective space may  readily be underestimated because of the {\em native} Euclidean geometric intuitions of such geometric transformations in one's mind.

First, since using homogeneous coordinates means pure algebraic representation in projective space, it is reasonable to expect that a {\em nice} definition to some specific geometric transformation is not only compatible with the conventional Euclidean intuitions but also independent of its Euclidean geometric background, i.e., such fundamental Euclidean concepts as {\em distances} would be invalid, {\em angles} have to be represented via cross ratio via Laguerre's formula~\cite[pp.342,409]{Perspectives}, and {\em Euclidean transformations} dependent on the validity of {\em distance} become undefined. Second, by saying {\em well-defined} or {\em nice} we mean a definition is able to solve the following two sides of a geometric transformation definition problem:  on the one side, given a homogeneous matrix, uniquely identify its geometric classification and characteristic features; on the other side, given sufficient characteristic geometric features, uniquely determine the homogeneous matrix of a geometric transformation.

Take reflection as an example. By definition in extended Euclidean space or in projective space, a reflection should be uniquely determined by its mirror hyperplane, which is a reference-coordinate-frame independent fact. While conventionally determining the homogeneous matrix of a reflection about an arbitrary mirror plane, {\em unnecessary} coordinate information has to be involved into concatenation multiplication matrix factors\cite[pp.34~\texttildelow~35, 47~\texttildelow~49]{CAD}\cite[pp.16~\texttildelow~19, 39]{Salomon}, i.e., coordinate information which finally cancels itself out has to be employed which makes the procedures of reflection determination neither unique nor straightforward. For many other geometric transformations such as central projection, parallel projection, scaling, shearing, central dilation, and so on, similar issues persist.

Such an approach in determining the homogeneous matrix of a transformation has the following drawbacks: (i)~The definition of a general  homogeneous reflection in projective space has a prerequisite that a homogeneous matrix, in diagonal form for the reflection case, has already been adopted as some kind of {\em standard reflection} based on the algebraic correspondences between the formulations and their Euclidean geometric meanings; (ii)~A series of Euclidean transformations $E_i$, the homogeneous matrices of which are actually algebraically undefined in projective geometry and which can be chosen almost arbitrarily, have to be employed based on an unarticulated truth(see theorem~\ref{zeroth}), which makes the procedure flawed and a little more complicate to program and code for many geometric transformations; (iii)~There is actually no algebraic definition in projective space to determine the geometric meaning of  homogeneous matrix obtained conversely without using such non-projective-geometry concepts as distance.

Chen\cite{investigation} proposed the concept stereohomology by taking advantage of an extended Desarguesian configuration in 2000 when trying to address the issue that the central and parallel projection matrices determination is dependent on the choice of coordinate system. However, no simple formulation for stereohomology was given in~\cite{investigation}.
As a natural continuation, Chen\cite{meaning} further represented stereohomology as 
 modified Householder elementary matrices~\cite[pp.1~\texttildelow~3]{Householder}. Central projection, parallel projection, translation, central symmetry, and reflection were included into stereohomology in \cite{meaning} and their potential applications in computer graphics were also proposed.

Equation \eqref{elementary matrix equation} is the algebraic form of elementary matrices, i.e., Householder elementary matrices,  which we confine to real number field $\mathbb{R}$ only in this paper. Matrices in this form were first introduced by Householder~\cite{Householder1958-1,Householder1958-2} and finalized as elementary in \cite[pp.1~\texttildelow~3]{Householder}. The famous Householder reflection successfully used in numerical analysis is the involutory and orthogonal case of Householder elementary matrices in equation \eqref{elementary matrix equation}.
\begin{align}\label{elementary matrix equation}
\nonumber {\displaystyle {\boldsymbol E}({\boldsymbol u},{\boldsymbol v}; \sigma )\mathop {=\!=} \limits^{def} {\boldsymbol I} -\sigma \cdot {\boldsymbol u}\cdot {\boldsymbol v}^{\scriptstyle \top} }\\
{\displaystyle {\boldsymbol u},{\boldsymbol v} \in \mathbb{R}^n
,\quad \sigma \in \mathbb{R}}
\end{align}

In this work, we will continue to address the definition and formulation issues in algebraic projective geometry for a series of geometric transformations commonly used in computer graphics by using the concept {\em stereohomology}\cite{investigation}, which is an extension of {\em homology} in \cite[p.60]{PointTransform}, {\em perspective collineation} in \cite[p.75]{Veblen} or {\em central collineaton} in\cite[pp.67~\texttildelow~73]{Ueberberg} to {\em n-}dimensional projective space, with degenerate cases considered in order to represent both nonsingular transformations and singular projections, and with prefix {\em stereo-} so as to distinguish it from the already existing {\em space homology}\cite[p.82]{PointTransform} and specify its potential applications in computer vision.

Our major contributions are:
we further extend the meaning of stereohomology in projective geometry, classify them into not only central projection, parallel projection, translation{~\cite{investigation}}, central symmetry, reflection{~\cite{meaning}}, but scaling, shearing, space elation, space homology and {\em direction}, which include all the possibilities of stereohomology, and formulate them into the coordinate system independent modified Householder elementary matrices.

\section{From stereohomology to Elementary Geometric Transformations}

For any geometric transformation $\mathscr{T}_0$ in square matrices, theorem \ref{zeroth} holds, which is also the theoretical basis for conventional approaches to homogeneous matrix construction of geometric transformations:
\begin{theorem}\label{zeroth}
Suppose $\mathscr{T}_0$ is a geometric transformation in projective space which transforms an arbitrary point $X$ into $Y$; and the homogeneous coordinates of $X$ and $Y$ in reference coordinate systems $(I)$ and $(II)$ are $(x)$,\;$(y)$,\;$(x')$,\;$(y')$\; respectively; the transformation matrices of $\mathscr{T}_0$ in $(I)$ and $(II)$ are ${\boldsymbol A}$ and ${\boldsymbol B}$ respectively, i.e., $(y)={\boldsymbol A}(x)$, $(y')={\boldsymbol B}(x')$; suppose the coordinate transformation from $(I)$ to $(II)$ is a nonsingular square matrix ${\boldsymbol T}$, i.e., $(x')={\boldsymbol T}(x)$, $(y')={\boldsymbol T}(y)$;  then:
\begin{align*}
(y')\,=\,{\boldsymbol B}(x')\,=\,{\boldsymbol T}\,(y)\,=\,{\boldsymbol T}\,{\boldsymbol A} \,(x)\,=\,{\boldsymbol T} \,{\boldsymbol A} \,{\boldsymbol T}^{-1} \,(x') \quad \forall X, Y \Rightarrow {\boldsymbol B}\,=\,{\boldsymbol T}\,{\boldsymbol A}\,{\boldsymbol T}^{-1}
\end{align*}
 The matrices of $\mathscr{T}_0$ in  $(I)$ and $(II)$  are similar.
\end{theorem}
Theorem \ref{zeroth} indicates that the invariants of matrices of a geometric transformation in different reference coordinate systems are their eigenvalues and the geometric and algebraic multiplicities thereof, which a reference-coordinate-system independent definition of the geometric transformation should depend on.

Though the scope of discussion in this paper will primarily be limited to point transformations, the following result on the relationship between a nonsingular hyperplane transformation and its point transformation counterpart in \cite[p.207]{planeTransform}, \cite[p.36]{Hartley} and \cite[pp.61,401]{Perspectives}, is useful for camera transformation, which is cited here as theorem \ref{hyperplane transformation}:

\begin{theorem}[Hyperplane transformation]\label{hyperplane transformation}Let $\mathscr{T}$ be a nonsinglular projective transformation with hyperplane transformation matrix ${\boldsymbol H}$ and point transformation matrix ${\boldsymbol P}$ in a reference coordinate system, then ${\boldsymbol H}={\boldsymbol P}^{-\top}$, where ${\boldsymbol P}^{-\top}$ is the transposed inverse of ${\boldsymbol P}$.
\end{theorem}

\subsection{Extension to Desargues' theorem and Desargues Configuration}

In some early projective geometry textbooks, there were already useful results for elementary geometric transformation representation.

For example, the discussion in \cite[pp.25~\texttildelow~28]{Mathews} implies that the two coplanar triangles which follow Desargues' theorm can determine a planar homology;  in \cite[pp.43~\texttildelow~44]{Veblen}, Desargues' theorem and its inverse were already extended to 3-space by two {\em perspective tetrahedra}, and space homology was also defined under the concept perspective collineation~\cite[pp.75~\texttildelow~76]{Veblen}.

Different from the results in~\cite{Mathews, Veblen}, a nonsingular collineation (or projective transformation)~\cite[p.xi]{PointTransform} is not enough in order to have singular projections included. So Desargues' theorem was extended to {\em n}-dimensional projective space with degenerate cases considered, and collineation (or projective transformation) was generalized to have singular cases\cite{investigation, meaning}.

\begin{theorem}[extended Desargues theorem{~\cite{Veblen,investigation,meaning}}]\label{Desargues}If in $n$-dimensional projective space ${\mathbb P}^n${\rm (2} $\leqslant$ $n$ $\in$ $\mathbb{Z}^{+}${\rm )}, the homogeneous coordinates of points $X_1,\,X_2,\,...,X_n,\,X_{n{\text +}1}$ have a rank of $n\!+\!1$, any $n$ of the $n{\text +}1$ points $Y_1,\,Y_2,\,...,Y_n,\,Y_{n{\text +}1}$ are linearly independent, and there exists a fixed point $S$ which is collinear to any two of the corresponding points: $X_i$ and $Y_i$ {\rm (} $i = 1, \cdots, n\!+\!1$ {\rm )}, {then the homogeneous coordinate vector set, which consists of $C_{n+1}^2$ intersection points defined as $S_{ij}$\,=$S_{j,i}$\,$\buildrel {de\!f} \over {=\!=}$\,$X_iX_j\cap Y_iY_j$ {\rm (}$i \neq j$, $i,j = \,1,\!\cdots\!, n\!+\!1${\rm )}, have a rank of $n$.}
\end{theorem}

When $n$=2, theorem \ref{Desargues} is analogous to the planar Desargues theorem; when $n$=3, theorem \ref{Desargues} is analogous to Theorem 2 in \cite[p.43]{Veblen}. Different from Desargues theorem and its 3-space extension in  \cite[pp.41,43~\texttildelow~44]{Veblen}, conversely the statement in theorem \ref{Desargues} is not true when the $n{\text +}1$ points $Y_1,\,Y_2,\,...,Y_n,\,Y_{n{\text +}1}$ are linearly dependent which is considered as the degenerate case of the statement in theorem \ref{Desargues}. There are literatures in different languages reporting the further extension of Desargues' theorm to even higher $n$-dimensional spaces, which are of less importance than the extension from {\em perspective triangles} to {\em perspective tetrahedra} since the results in 2 and 3 dimensional spaces are enough for applications in vision.

\begin{definition}[Extended Desarguesian configuration]\label{e-configuration} A set of $\;2n\!+\!3\;$ points, $X_1$, $X_2$, $\cdots$\;, $X_n$, $X_{n{\text +}1}$, $S$, $Y_1$, $Y_2$, $\cdots$\;, $Y_n$, $Y_{n{\text +}1}$ in theorem \ref{Desargues} is an extended Desarguesian configuration, denoted as:
\begin{equation} \label{configuration}
X_1\,X_2\cdots X_n\,X_{n{\text +}1}\text{-}S\text{-}Y_1 Y_2\cdots Y_n\,Y_{n{\text +}1}
\end{equation}
\end{definition}

A series of extended Desargues configurations have been visualized in figures~\ref{fig:Desarguesian configuration projection}~\texttildelow~\ref{fig:Desarguesian configuration translation} to illustrate how they are being used in defining {\em elementary geometric transformations}, which are useful to verify the {\em compatibility with} of these transformations with their Euclidean counterparts in geometric meaning. Different from the conventional Desargues configuration, the configuration in figures~\ref{fig:Desarguesian configuration projection}~\texttildelow~\ref{fig:Desarguesian configuration parallel projection},  are for singular {\em projections}, therefore the $Y_1$, $Y_2$, $Y_3$ and $Y_4$ in them are coplanar and no three of the four are collinear.

In order to include singular projections into the {\em elementary} representation framework, we need to have singular geometric transformations included besides the nonsingular general projection transformations.

Similar to the definition of collineation\cite[pp.xi,6]{PointTransform}, a generalized collineation or genralized projective transformation can be defined as:
\begin{definition}[Generalized projective transformation\cite{investigation,meaning}]\label{g-projective}In ${\mathbb P}^n$, if ($n\!+\!1$)-square matrix ${\boldsymbol T}=(t_{i,j})$ in Equation \eqref{generalized collieation} ($\rho$ $\in$ $\mathbb{R}$) defines a geometric transformation $\mathscr{T}$ , and rank(${\boldsymbol T}$)$\geqslant n$, then $\mathscr{T}$ with matrix ${\boldsymbol T}$ is a generalized collineation, also called generalized projective transformation.

{\begin{eqnarray}\label{generalized collieation} \left\{
\begin{array}{c@{\extracolsep{6pt}}c@{\extracolsep{6pt}}c@{\extracolsep{6pt}}c@{\extracolsep{6pt}}c@{\extracolsep{6pt}}c@{\extracolsep{6pt}}c@{\extracolsep{6pt}}c}
   \rho x'_1& = &t_{1,1} x_1 & +& t_{1,2} x_2 & + \cdots & +& t_{1,n{\text +1}} x_{n{\text +1}}   \\
   \rho x'_2& = &t_{2,1} x_1 & +& t_{2,2} x_2 & + \cdots & + &t_{2,n{\text +1}} x_{n{\text +1}}   \\
     \vdots &  & \vdots & & \vdots & \ddots & & \vdots  \\
   \rho x'_{n{\text +1}}  & = &t_{n{\text +1,1}} x_1 & +& t_{n{\text +1,2}} x_2 & + \cdots & +& t_{n{\text +1,}n{\text +1}} x_{n{\text +1}}
\end{array}\right.
\end{eqnarray}}
\end{definition}
\subsection{Definition of Stereohomology}

\begin{definition}[Stereohomology\cite{investigation,meaning}]\label{stereohomology} In ${\mathbb P}^n$, a generalized projective transformation $\mathscr{T}$ with ($n\!+\!1$)-square matrix ${\boldsymbol T}$ is a stereohomology, ($1$) if there exists a rank $n$ hyperplane ( {\it stereohomology hyperplane}, denoted as $\pi$ ) any point on which is an eigenvector of ${\boldsymbol T}$; and ($2$) if there exists a fixed point ({\it stereohomology center}, denoted as $S$ ) which is collinear with any pair of corresponding points through $\mathscr{T}$.
\end{definition}

According to theorems 5.3 and 5.6 in \cite[pp.68~\texttildelow~73]{Ueberberg}, the existence of a fixed hyperplane implies the existence of a fixed center and vice versa. The results hold for generalized collineation. Since ``(1)" and ``(2)" in definition \ref{stereohomology} can be derived from each other, the definition of stereohomology in \ref{stereohomology} can hence be simplified to have ``(1)" or ``(2)" only. 

In order to algebraically identify whether a square matrix is a stereohomology or not, ``(1)" in definition \ref{stereohomology} plays a specially important role since it implies that any ({\em n}+1)-dimensional matrix which has an eigenvalue with geometric multiplicity of {\em n} is a stereohomology in ${\mathbb P}^n$.

Similar to the definition of perspective collineation in \cite[p.75]{Veblen} and that of homology in \cite[p.60]{PointTransform}, stereohomology by definition has its stereohomology hyperplane $\pi$ (axis in ${\mathbb P}^2$) and center $S$. Stereohomology was first defined as {\em perspective} in \cite{investigation} then {\em elementary perspective} in \cite{meaning} when $s$ $\in$ $\pi$, {\em homology} in \cite{investigation} and {\em elementary homology} in \cite{meaning} when $s$ $\notin$ $\pi$. We will follow the definitions in \cite{meaning} in the next subsection before we get a full picture of stereohomology.

\begin{lemma}[Existence \& uniqueness theorem \cite{investigation, meaning}]\label{existence and uniqueness theorem}
 There exists a unique generalized projective transformation $\mathscr{T}$ which transforms $X_1$,$\,X_2$,$\cdots$ $X_n$,\,$X_{n+1}$ and $S$ in an extended Desarguesian Configuration (Eq.\eqref{configuration} into $Y_1$,$\,Y_2$,$\cdots$ $Y_n$,\,$Y_{n+1}$ and $S$(or \it null\rm) respectively.
\end{lemma}
 The proof in 3-space was constructively presented in \cite{investigation} via Gramer's rule, which is analogous to that of the fundamental theorem on the existence and uniqueness of {\em projective transformation} in projective geometry ({\color{olive}See Appendices}).

If the homogeneous coordinates of points in an extended Desarguesian configuration $ABCD\text{-}S\text{-}A'B'C'D'$ in ${\mathbb P}^3$ are
\[{\small \begin{array}{*{20}l}
   {A:(a_1 ,a_2 ,a_3 ,a_4 )^\top} & {B:(b{}_1,b_2 ,b_3 ,b_4 )^\top} & {C:(c_1 ,c_2 ,c_3 ,c_4 )^\top} &  {D:(d_1 ,d_2 ,d_3 ,d_4 )^\top} & \\[8pt]
   {A':(a'_1 ,a'_2 ,a'_3 ,a'_4 )^\top}  & {B':(b'_1,b'_2 ,b'_3 ,b'_4 )^\top} & {C':(c'_1 ,c'_2 ,c'_3 ,c'_4 )^\top} &    {D':(d'_1 ,d'_2 ,d'_3 ,d'_4 )^\top}
 \end{array}  \!\!\!
{S:(s_1 ,s_2 ,s_3 ,s_4 )^\top} }
\]
\noindent respectively, denote:

\[{\small
\begin{gathered}
  \Delta _1  = \left| {\begin{array}{*{20}c}
   {s_1 } & {b_1 } & {c_1 } & {d_1 }  \\[1pt]
   {s_2 } & {b_2 } & {c_2 } & {d_2 }  \\[1pt]
   {s_3 } & {b_3 } & {c_3 } & {d_3 }  \\[1pt]
   {s_4 } & {b_4 } & {c_4 } & {d_4 }  \\[1pt]

 \end{array} } \right|,\:\Delta _2  = \left| {\begin{array}{*{20}c}
   {a_1 } & {s_1 } & {c_1 } & {d_1 }  \\[1pt]
   {a_2 } & {s_2 } & {c_2 } & {d_2 }  \\[1pt]
   {a_3 } & {s_3 } & {c_3 } & {d_3 }  \\[1pt]
   {a_4 } & {s_4 } & {c_4 } & {d_4 }  \\[1pt]

 \end{array} } \right|,
  \Delta _3  = \left| {\begin{array}{*{20}c}
   {a_1 } & {b_1 } & {s_1 } & {d_1 }  \\[1pt]
   {a_2 } & {b_2 } & {s_2 } & {d_2 }  \\[1pt]
   {a_3 } & {b_3 } & {s_3 } & {d_3 }  \\[1pt]
   {a_4 } & {b_4 } & {s_4 } & {d_4 }

 \end{array} } \right|, \Delta _4  =  \left| {\begin{array}{*{20}c}
   {a_1 } & {b_1 } & {c_1 } & {s_1 }  \\[1pt]
   {a_2 } & {b_2 } & {c_2 } & {s_2 }  \\[1pt]
   {a_3 } & {b_3 } & {c_3 } & {s_3 }  \\[1pt]
   {a_4 } & {b_4 } & {c_4 } & {s_4 }

 \end{array} } \right|, \hfill \\
\end{gathered}
}\]

\[{\small
\begin{gathered}
  \Delta '_1  = \left| {\begin{array}{*{20}c}
   {s_1 } & {b'_1 } & {c'_1 } & {d'_1 }  \\[1pt]
   {s_2 } & {b'_2 } & {c'_2 } & {d'_2 }  \\[1pt]
   {s_3 } & {b'_3 } & {c'_3 } & {d'_3 }  \\[1pt]
   {s_4 } & {b'_4 } & {c'_4 } & {d'_4 }  \\[1pt]

 \end{array} } \right|,\:\Delta '_2  = \left| {\begin{array}{*{20}c}
   {a'_1 } & {s_1 } & {c'_1 } & {d'_1 }  \\[1pt]
   {a'_2 } & {s_2 } & {c'_2 } & {d'_2 }  \\[1pt]
   {a'_3 } & {s_3 } & {c'_3 } & {d'_3 }  \\[1pt]
   {a'_4 } & {s_4 } & {c'_4 } & {d'_4 }  \\[1pt]

 \end{array} } \right|,
  \Delta '_3  = \left| {\begin{array}{*{20}c}
   {a'_1 } & {b'_1 } & {s_1 } & {d'_1 }  \\[1pt]
   {a'_2 } & {b'_2 } & {s_2 } & {d'_2 }  \\[1pt]
   {a'_3 } & {b'_3 } & {s_3 } & {d'_3 }  \\[1pt]
   {a'_4 } & {b'_4 } & {s_4 } & {d'_4 }

 \end{array} } \right|,\:\Delta '_4  = \left| {\begin{array}{*{20}c}
   {a'_1 } & {b'_1 } & {c'_1 } & {s_1 }  \\[1pt]
   {a'_2 } & {b'_2 } & {c'_2 } & {s_2 }  \\[1pt]
   {a'_3 } & {b'_3 } & {c'_3 } & {s_3 }  \\[1pt]
   {a'_4 } & {b'_4 } & {c'_4 } & {s_4 }

 \end{array} } \right|, \hfill \\
\end{gathered}
}\]

\noindent then a generalized collineation satisfies lemma \ref{existence and uniqueness theorem} obtained via Gramer's rule is as in equation \eqref{initial-matrix} \cite{investigation,meaning}:
\begin{equation}{\small \displaystyle \label{initial-matrix}
 \mathscr{T}^{3d}\:=\: k \cdot \left[ {\begin{array}{*{20}c}
   {a'_1 \Delta '_1 } & {b'_1 \Delta '_2 } & {c'_1 \Delta '_3 } & {d'_1 \Delta '_4 }  \\[10pt]
   {a'_2 \Delta '_1 } & {b'_2 \Delta '_2 } & {c'_2 \Delta '_3 } & {d'_2 \Delta '_4 }  \\[10pt]
   {a'_3 \Delta '_1 } & {b'_3 \Delta '_2 } & {c'_3 \Delta '_3 } & {d'_3 \Delta '_4 }  \\[10pt]
   {a'_4 \Delta '_1 } & {b'_4 \Delta '_2 } & {c'_4 \Delta '_3 } & {d'_4 \Delta '_4 }
 \end{array} } \right] \cdot \left[ {\begin{array}{*{20}c}
   {a_1 \Delta _1 } & {b_1 \Delta _2 } & {c_1 \Delta _3 } & {d_1 \Delta _4 }  \\[10pt]
   {a_2 \Delta _1 } & {b_2 \Delta _2 } & {c_2 \Delta _3 } & {d_2 \Delta _4 }  \\[10pt]
   {a_3 \Delta _1 } & {b_3 \Delta _2 } & {c_3 \Delta _3 } & {d_3 \Delta _4 }  \\[10pt]
   {a_4 \Delta _1 } & {b_4 \Delta _2 } & {c_4 \Delta _3 } & {d_4 \Delta _4 }
 \end{array} } \right]^{ - 1} \quad  \forall \; 0 \neq k \in \mathbb{R}
}\end{equation}

In \cite{investigation}, it was proved that a generalized collineation matrix thus obtained has an eigenvalue with geometric multiplicity of 3 ({\color{olive}See Appendices}). The results can be analogously extended to ${\mathbb P}^n$. Therefore it is a stereohomology.

Hence theorem \ref{e-stereohomolgy matrix} will be straightforward:
\begin{theorem}[Existence \& Uniqueness of stereohomology\cite{investigation,meaning}]\label{e-stereohomolgy matrix} A stereohomology can be uniquely determined by an extended Desarguesian configuration: the unique generalized projective transformation matrix which transforms $X_1$,\,$X_2$,$\cdots$ $X_n,$ $\,X_{n+1}$ and $S$ in extended Desarguesian configuration {\rm Eq.} \eqref{configuration} in Definition \ref{e-configuration} into $Y_1$,\,$Y_2$,$\cdots Y_n,$ $\,Y_{n+1}$ and $S$ {\rm (} or $Null$ {\rm )} respectively. \end{theorem}
\begin{proof}
{\color{olive}See Appendices}.
\end{proof}

\subsection{Elementary Matrix Representation}

Stereohomology obtained in equation \eqref{initial-matrix} is rather complicated for real application. So it is natural to consider the possibility of simplification, which was realized by symbolic computation\cite{meaning}.

Suppose an {\em elementary homology} $\mathscr{T}^{3d}$ in ${\mathbb P}^3$ , of which the stereohomology center $S$ in homogeneous is ($s_1,s_2,s_3,s_4$)$^{\scriptscriptstyle \top}$, and the stereohomology hyperplane  $\pi$'s homogeneous coordinate is ($a,b,c,d$)$^{\scriptscriptstyle \top} $; Suppose none of any two of $a$, $b$, $c$, and $d$ are zero concurrently, and suppose the two eigenvalues of $\mathscr{T}^{3d}$ are  $\lambda$, corresponding to the eigenspace $\pi$ with a geometric mutiplicity of 3, and $\rho$ corresponding to $S$.

Hence, ($s_1,s_2,s_3,s_4$)$^{\scriptscriptstyle \top}$ is an associated eigenvector with $\rho$, and the linearly independent $(-b,a,0,0)^{\scriptscriptstyle \top}$, ($-c,0,a,0$)$^{\scriptscriptstyle \top}$,  ($-d,0,0,a$)$^{\scriptscriptstyle \top}$ are the associated eigenvectors with $\lambda$,  i.e.:

\begin{equation}\label{eigenvector-group} \left\{
{\begin{array}{*{20}c}
    \mathscr{T}^{3d} & \cdot&(s_1 ,s_2 ,s_3 ,s_4 )^{\scriptscriptstyle\top}& =& \rho & \cdot &(s_1 ,s_2 ,s_3 ,s_4 )^{\scriptscriptstyle \top}  \\
    \mathscr{T}^{3d} & \cdot&( - b, a, 0, 0)^{\scriptscriptstyle\top}&=& \lambda & \cdot &( - b, a, 0, 0)^{\scriptscriptstyle \top}  \\
    \mathscr{T}^{3d} & \cdot&( - c, 0, a, 0)^{\scriptscriptstyle\top}&=& \lambda & \cdot &( - c, 0, a, 0)^{\scriptscriptstyle \top}  \\
    \mathscr{T}^{3d} & \cdot&( - d, 0, 0, a)^{\scriptscriptstyle\top}&=&\lambda & \cdot &( - d, 0, 0, a)^{\scriptscriptstyle \top}  \\
 \end{array} } \right.
\end{equation}

Equation \eqref{eigenvector-group} is equivalent to:
\begin{equation*}
\mathscr{T}^{3d}= \left[ {\begin{array}{*{20}c}
   { - \lambda b} & { - \lambda c} & { - \lambda d} & {\rho s_1 }  \\[6pt]
   {\lambda a} & 0 & 0 & {\rho s_2 }  \\[6pt]
   0 & {\lambda a} & 0 & {\rho s_3 }  \\[6pt]
   0 & 0 & {\lambda a} & {\rho s_4 }  \\

 \end{array} } \right] \cdot \left[ {\begin{array}{*{20}c}
   { - b} & { - c} & { - d} & {s_1 }  \\[6pt]
   a & 0 & 0 & {s_2 }  \\[6pt]
   0 & a & 0 & {s_3 }  \\[6pt]
   0 & 0 & a & {s_4 }  \\[6pt]
 \end{array} } \right]^{ - 1} \buildrel {\text Symbolic\; computation } \over {=\!\!=\!\!=\!\!=\!\!=\!\!=\!\!=\!\!=\!\!=\!\!=\!\!=\!\!=\!\!=\!=\!\!=\!\!=\!\!=}
\end{equation*}

\begin{equation} \label{symbolicsolution}{\small \frac{1}{{as_1 \!\! + \!\! bs_2 \!\! + \!\!
cs_3 \!\! + \!\! ds_4 }}  \cdot }
 \left[ \!\!\! {\footnotesize \begin{array}{c@{\extracolsep{-18pt}}c@{\extracolsep{-18pt}}c@{\extracolsep{-18pt}}c}
   {\rho as_1 \! + \! \lambda bs_2 \! + \! \lambda cs_3  + \lambda ds_4 } & {bs_1 (\rho  - \lambda )} & {cs_1 (\rho  - \lambda )} & {ds_1 (\rho  - \lambda )}  \\[10pt]
   {as_2 (\rho  - \lambda )} & {\lambda as_1 \! +  \! \rho bs_2 \! + \! \lambda cs_3 \! + \! \lambda ds_4 } & {cs_2 (\rho  - \lambda )} & {ds_2 (\rho  - \lambda )}  \\[10pt]
   {as_3 (\rho  - \lambda )} & {bs_3 (\rho  - \lambda )} & {\lambda as_1 \! + \! \lambda bs_2 \! + \! \rho cs_3 \! + \! \lambda ds_4 } & {ds_3 (\rho  - \lambda )}  \\[10pt]
   {as_4 (\rho  - \lambda )} & {bs_4 (\rho  - \lambda )} & {cs_4 (\rho  - \lambda )} & {\lambda as_1 \! + \! \lambda bs_2 \! + \! \lambda cs_3 \! + \! \rho ds_4 }  \\[10pt]
\end{array}} \!\!\! \right]
\end{equation}

Equation \eqref{symbolicsolution} can be rewritten and extended to ${\mathbb P}^n$ as:

\begin{align}\label{elementary homology}
\nonumber \mathscr{T}(\ (s),(\pi) ; \lambda, \rho )\:\mathop
{=\!\!\!=\!\!\!=} \limits^{\ de\!f} \;\lambda \cdot {\boldsymbol I} + (\rho  - \lambda)
\cdot\frac{{(s)\cdot(\pi) ^{\scriptscriptstyle \top} }} {{(s)^{\scriptscriptstyle \top} \cdot (\pi) }}\\
(s), (\pi) \in {\mathbb P}^n,\quad (s)^{\scriptscriptstyle \top}\cdot(\pi) \neq 0, \quad \lambda, \rho \in \mathbb{R}\qquad
\end{align}

Similar form of {\em elementary perspective} can be obtained as:

\begin{align}\label{elementary perspective} \nonumber
\mathscr{T}((s),(\pi) ;\mu )\mathop {=\!=} \limits^{def} {\boldsymbol I} + \frac{ \mu \cdot (s)\cdot{(\pi)}^{\scriptscriptstyle \top} } {\sqrt {(s)^{\scriptscriptstyle \top}\!\! \cdot\! (s)\!\cdot\!(\pi) ^{\scriptscriptstyle \top}\!\!\cdot\! (\pi) } }\\
(s), (\pi)  \in {\mathbb P}^n,  \quad (s)^{\scriptscriptstyle \top}\!\! \cdot\! (\pi)\!\! = 0, \quad 0 \neq \mu \in \mathbb{R}\;
\end{align}

In equation \eqref{elementary perspective}, denominator ${\sqrt {(s)^{\scriptscriptstyle \top}\! \cdot\! (s)\!\cdot\!(\pi) ^{\scriptscriptstyle \top}\!\cdot\! (\pi) } }$  is used so that $\mu$ is independent of the particular homogeneous coordinates choice of ($s$) and ($\pi$). The elementary matrices in equations \eqref{elementary homology} and \eqref{elementary perspective}, which are de facto equivalent to Householder's elementary matrices in equation \eqref{elementary matrix equation}.

It has been proved that both equations \eqref{elementary homology} and \eqref{elementary perspective} are stereohomology in ${\mathbb P}^n$ \cite{meaning} by using theorem 1.3.20 in \cite[pp.53~\texttildelow~54]{Horn} which is cited as lemma \ref{sylvester} here:

\begin{lemma}\label{sylvester}
If a matrix ${\boldsymbol A}$ $\in$ $\mathbb{F}^{m\times n}$, ${\boldsymbol B}$ $\in$ $\mathbb{F}^{n\times m}$, and the characteristic polynomials of ${\boldsymbol A{\boldsymbol B}}$ and ${\boldsymbol B{\boldsymbol A}}$ are $f_{AB}(\lambda)$ and $f_{BA}(\lambda)$ respectively, then:

\begin{equation}\label{sylvester-lambda}
f_{AB}(\lambda)=\lambda^{m-n}\cdot f_{BA}(\lambda)
\end{equation}

\end{lemma}

\begin{proposition}[Property of elementary matrices]\label{rank of stereohomology} The rank of $n$+1 dimensional elementary matrices $\mathscr{T}$ in equations \eqref{elementary homology} and \eqref{elementary perspective} is equal to or greater than $n$.
\end{proposition}
\begin{proof}
Let ${\boldsymbol A}$ $=\alpha\cdot (s)$, ${\boldsymbol B}$ $=(\pi)\!^\top$\; ($\forall$ $\alpha$ $\in$ $\mathbb{R}$). Applying equation.\:\eqref{sylvester-lambda} in  {\em Lemma}\:\ref{sylvester} to the characteristic polynomials of ${\boldsymbol A{\boldsymbol B}}$ and ${\boldsymbol B{\boldsymbol A}}$, we have the following results:

When $S\notin \pi$, for equation \eqref{elementary homology}, let $$\alpha=-\frac{\rho-\lambda}{(s)\!^{\scriptscriptstyle \top}\!\!\cdot\! (\pi)}, $$ then we have the determinant of $\mathscr{T}$:
\begin{align}
\nonumber {\rm det}(\mathscr{T})&={\rm det}\left( \lambda \cdot {\boldsymbol I}+(\rho - \lambda)\frac{(s)\!\cdot (\pi
)\!^{\scriptscriptstyle \top}}{(s)\!^{\scriptscriptstyle \top}\!\!\cdot\!(\pi )}\right) = f_{AB} (\lambda)=\lambda^{n} \cdot f_{BA}(\lambda)=\lambda^{n}\cdot \rho
\end{align}
Since $\lambda$ $\neq$ 0 is the eigenvalue with the geometric multiplicity of $n$, if and only if $\rho$\, =\, 0, rank($\mathscr{T}$)\,=\,$n$, otherwise rank($\mathscr{T}$)\,=\,$n$+1.

When $S$ $\in$ $\pi$, for equation \eqref{elementary perspective}, let $$\alpha = \frac{-\mu}{\sqrt {(s)^{\scriptscriptstyle \top}\!\!\cdot\!(s)\cdot(\pi )^{\scriptscriptstyle \top}\!\!(\pi )}}, $$ then the determinant of $\mathscr{T}$:
\begin{align}
\nonumber {\rm det}(\mathscr{T})&={\rm det}\left( \lambda \cdot {\boldsymbol I}+\mu \cdot \frac{(s)\!^{\scriptscriptstyle \top}\!\!\cdot (\pi )}{\sqrt{(s)^{\scriptscriptstyle \top}\!\!\cdot\!(s)\cdot(\pi )^{\scriptscriptstyle \top}\!\!(\pi )}}\right) = f_{AB} (\lambda)=\lambda^{n}\cdot f_{BA}(\lambda)=\lambda^{n+1}
\end{align}
Since $\lambda$ is the only nonzero eigenvalue of $\mathscr{T}$,  an elementary perspective $\mathscr{T}$ always is nonsingular. Hence rank($\mathscr{T}$) $=$ $n$+1.
\end{proof}

So far we have successfully represented stereohomology into modified Householder elementary matrices in equations \eqref{elementary homology} and \eqref{elementary perspective}. This also makes it possible a stereohomology matrix can be determined by a fixed hyperplane $\pi$, a fixed point $S$ and the eigenvalues with geometric meaning independent of the particular choice of coordinate system. Conversely, given a square matrix, it is easy to determine whether it is a stereohomology by its checking its eigen-decomposition.

\begin{theorem}[Three stereohomology theorem\cite{meaning}]\label{tri-sterehomology}
If $\mathscr{T}_1$ and $\mathscr{T}_2$ are elementary geometric transformations with stereohomology centers of $S_1$, $S_2$ , and stereohomology hyperplanes of $\pi_1$ and $\pi_2$ respectively. $\mathscr{T}_3$ $=$ $\mathscr{T}_1$ $\cdot$ $\mathscr{T}_2$. Then:
\begin{description}
\item [{\small (i)}] If $S_1$ coincides with $S_2$, then $\mathscr{T}_3$ is also a stereohomology; denote the stereohomology hyperplane of $\mathscr{T}_3$ as $\pi_3$, then $\pi_1$, $\pi_2$ and $\pi_3$ are collinear.
\item [{\small (ii)}]If $\pi_1$ coincides with $\pi_2$, then $\mathscr{T}_3$ is also a stereohomology; denote the stereohomology center of $\mathscr{T}_3$ as $S_3$, then $S_1$, $S_2$ and $S_3$ are collinear.
\end{description}
\end{theorem}

This can be proved by using the result that elementary matrices in equations \eqref{elementary homology} and \eqref{elementary perspective} are stereohomology by definition \ref{stereohomology}.

\section{Unification of Elementary Geometric Transformations: stereohomology}

In \cite{investigation}, only central projection, parallel projection, translation and central dilation were defined as stereohomology since only singlularity of the transformation was considered. When involution was taken into consideration, reflection and central symmetry were both added in \cite{meaning}.

\subsection{Definitions and Representation of Elementary Geometric Transformations }
In this section, we discuss more possibilities of $S$ and $\pi$ in the extended Desarguesian configuration of a stereohomology, and define all the following geometric transformations into the stereohomology family as {\em elementary}: central projection, parallel projection (with both oblique and orthographic cases), direction, space homology(slightly different from those in \cite[p.75]{Veblen} and \cite[p.60]{PointTransform}), scaling(with both oblique and orthographic cases, the involutory case of which is a reflection), central dilation(the involutory case of which is central symmetry), space elation, shearing and translation.

Note that all the specific elementary geometric transformations are classified and defined strictly based on the geometric features in the corresponding {\em extended Desarguesian configuration}, e.g., figures~\ref{fig:Desarguesian configuration reflection} and~\ref{fig:Desarguesian configuration projection} for reflection and central projection respectively. We omit those figures for other elementary geometric transformations and only list their algebraic descriptions in table~\ref{classification table}. The definitions thus proposed also present an approach in determining whether a given square matrix is a specific elementary geometric transformation or not.

Table~\ref{classification table} summarizes all these elementary geometric transformations defined as stereohomology which will be discussed in details respectively. Before going on for further discussion, it is necessary to present the following properties of stereohomology as propositions \ref{singular stereohomology} and \ref{involutory stereohomology}, which are the basis for some of the definitions in table \ref{classification table} and both can be proved by using equations \eqref{elementary homology} and \eqref{elementary perspective}.

\def\arraystretch{1.7}
\begin{sidewaystable}
\centering
\caption{\large Classification and Definitions of Geometric Transformations Which are Stereohomology}
\label{classification table}
\resizebox{23cm}{!} { 
 \begin{tabular}[t]{>{\centering\arraybackslash}m{.2in}cccLLr}
\hlinewd{2pt}
\; No.\; & \tc{$S$ vs. $\pi$} & \tc{ $\begin{casesBig} {\textrm{ Transformation}}{\textrm { matrix property}}\end{casesBig}$} & \tc{Property of $\pi$} & \tc{ Property of $S$ }& \tc{Transformation matrix formula} & \tc{Definition of transformation} \tabularnewline
\hline
1& $S\notin\pi$ & \tc{Singular} & \tc{Ordinary} & \tc{Ordinary} & & Central Projection {\color{red}\checkmark}\tabularnewline
2& $S\notin\pi$ & \tc{Singular} &\tc{Ordinary} & \tc{Infinite} &  \[\mathscr{T}\left ((s),(\pi); \lambda \right )=
 \lambda\cdot{\boldsymbol I} -\lambda\cdot \frac{{(s)\cdot(\pi) ^{\scriptscriptstyle \top}}}{{(s)^{\scriptscriptstyle \top}  \cdot (\pi)} } \]  & Oblique \& Orthographic Parallel Projection{\color{red} \checkmark} \tabularnewline
3& $S\notin\pi$ & \tc{Singular} & \tc{Infinite} & \tc{Ordinary} & & Direction {\color{red} \checkmark} \tabularnewline
\hline
4& $S\notin\pi$ & \tc{Nonsingular} & \tc{Ordinary} & \tc{Ordinary} & & {Space homology}{\color{red} \checkmark}\tabularnewline
5& $S\notin\pi$ & \tc{Nonsingular} & \tc{Ordinary} & \tc{Infinite} & \[\mathscr{T}(\ (s),(\pi) ;\rho, \lambda)\:\mathop =\;{\lambda\cdot\boldsymbol I} + (\rho - \lambda) \cdot \frac{{(s)\cdot(\pi) ^{\scriptscriptstyle \top} }} {{(s)^{\scriptscriptstyle \top} \cdot (\pi) }} \] & Oblique \& Orthographic Elementary Scaling {\color{red} \checkmark} \tabularnewline
6& \tc{$S\notin\pi$} & \tc{Nonsingular} & \tc{Infinite} & \tc{Ordinary} & & Central Dilation {\color{red} \checkmark} \tabularnewline
\hline
7& $S\notin\pi$ & \tc{$\begin{casesBig} {\textrm{ Nonsingular}}{\textrm {\&Involutory}}\end{casesBig}$} & \tc{Ordinary} & \tc{Ordinary} & & {Involutory space homology} {\color{red} \checkmark}\tabularnewline
8& $S\notin\pi$ & \tc{$\begin{casesBig} {\textrm {Nonsingular}}{\textrm {\&Involutory}}\end{casesBig}$} & \tc{Ordinary} & \tc{Infinite} & \[\mathscr{T}\left ( (s),(\pi); \lambda \right )=
 \lambda\cdot{\boldsymbol I} - 2\lambda \cdot {\huge \frac{(s)\cdot(\pi)\! ^{\scriptscriptstyle\top}} {(s)\!^{\scriptscriptstyle\top} \! \cdot (\pi) }} \] & Skew(Oblique) \& Orthographic Reflection {\color{red} \checkmark}\tabularnewline
9& $S\notin\pi$ & \tc{$\begin{casesBig} {\textrm {Nonsingular}}{\textrm {\&Involutory}}\end{casesBig}$} & \tc{Infinite} & \tc{Ordinary} & & Central Symmetry {\color{red} \checkmark}\tabularnewline
\hline
10& $S\in\pi$ & \tc{Nonsingular} & \tc{Ordinary} & \tc{Ordinary} & & {Space elation}{\color{red} \checkmark}\tabularnewline
11& $S\in\pi$ & \tc{Nonsingular} & \tc{Ordinary} & \tc{Infinite} & \[ \mathscr{T}((s),(\pi) ;\lambda,\mu )=
{\lambda\cdot\boldsymbol I} + \frac{ \mu \cdot
(s)\cdot{(\pi)}^{\scriptscriptstyle \top} } {\sqrt
{(s)^{\scriptscriptstyle \top}\!\!
\cdot\! (s)\!\cdot\!(\pi) ^{\scriptscriptstyle \top}\!\!\cdot\! (\pi) } }\]& Shearing {\color{red} \checkmark} \tabularnewline
12& $S\in\pi$ & \tc{Nonsingular} & \tc{Infinite} & \tc{Infinite} & & Translation {\color{red} \checkmark} \tabularnewline
\hlinewd{1pt}
\end{tabular}
}  
\end{sidewaystable}

\begin{proposition}[singular stereohomology]\label{singular stereohomology} The transformation matrix $\mathscr{T}$ of a singular stereohomology can and only can be represented as(without loss of generality, $\lambda$ in equation \eqref{elementary homology} is fixed at {\rm 1}):
\begin{equation}\label{singular stereohomology equation}
\mathscr{T}\left ( (s),(\pi) \right )\mathop {=\!\!=\!\!=} \limits^{def} {\boldsymbol I} - \frac{{(s)\cdot(\pi)\! ^\top}} {{(s)\!^\top \!\! \cdot (\pi) }}
\end{equation}
where $(s)$,$(\pi)$ $\in$ $\mathbb{P}^n$,\; $(s)\!^\top$ $\!\!\cdot$ $(\pi)$ $\neq$ 0.
\end{proposition}

\begin{definition}[Involutory]\label{involutory}
If a projective geometric transformation $\mathscr{T}$ satisfies:
$$\mathscr{T}^2 = k \cdot {\boldsymbol I}, \quad \exists \; 0 \neq k \in \mathbb{R}$$
then $\mathscr{T}$ is involutory, or is called an involutory (projective) transformation.
\end{definition}

\begin{proposition}[involutory stereohomology]\label{involutory stereohomology}
An involutory stereohomology $\mathscr{T}$ can be represented as ( without loss of generality, still $\lambda$ in equation \eqref{elementary homology} is fixed at 1 ):
\end{proposition}
\begin{equation} \label{involutory stereohomology equation}
{\displaystyle \mathscr{T}\left( (s),(\pi) \right )\mathop {=\!\!=\!\!=} \limits^{def} {\boldsymbol I} -
2\cdot\frac{{(s)\cdot(\pi) ^{\scriptscriptstyle \top} }} {{(s)^{\scriptscriptstyle \top}\!\cdot (\pi) }}}
\end{equation}

Similar reflection formula on projective involution was also given in \cite[p.415, theorem 21.4]{Perspectives} which can be considered as a special case of stereohomology.

\begin{definition}[Central projection]\label{central projection}
A central projection $\mathscr{T}$ is a singular stereohomology, of which both the stereohomology center $S$ and the stereohomology hyperplane $\pi$ are ordinary (or finite).  See figure~\ref{fig:Desarguesian configuration projection} for its extended Desarguesian configuration illustration. The stereohomology center $S$ is  the {\em projection center} of $\mathscr{T}$, and $\pi$ is the {\em projection hyperplane} or {\em image hyperplane} of $\mathscr{T}$.
\end{definition}

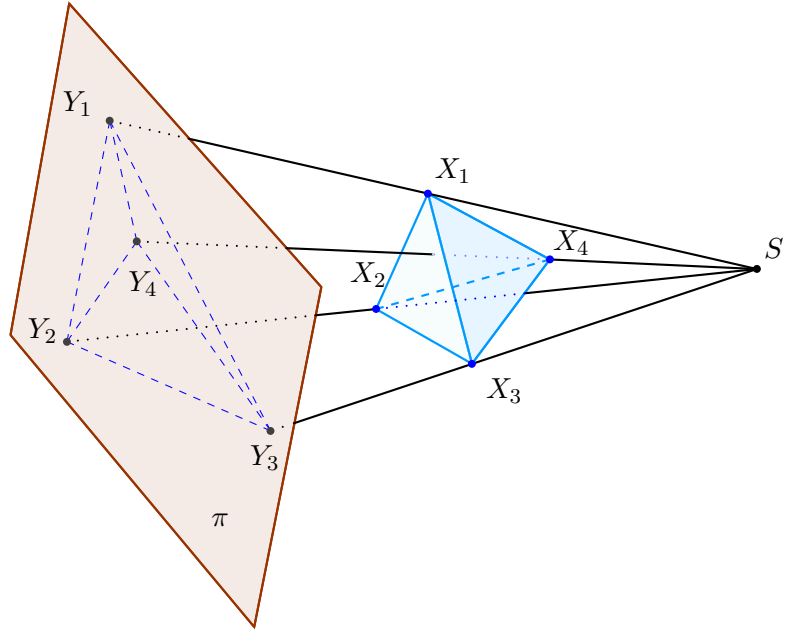
\begin{figure}[!hpt]
\begin{center}
\newrgbcolor{afeeee}{0.69 0.93 0.93}
\newrgbcolor{qqzzff}{0 0.6 1}
\newrgbcolor{zzttqq}{0.6 0.2 0}
\newrgbcolor{uququq}{0.25 0.25 0.25}
\newrgbcolor{bcduew}{0.74 0.83 0.9}
\newrgbcolor{xdxdff}{0.49 0.49 1}
\psset{xunit=0.75cm,yunit=0.75cm,algebraic=true,dimen=middle,dotstyle=o,dotsize=3pt 0,linewidth=0.8pt,arrowsize=3pt 2,arrowinset=0.25}
\begin{pspicture*}(3.61,-5.85)(18.09,5.96)
\pspolygon[linecolor=afeeee,fillcolor=afeeee,fillstyle=solid,opacity=0.1](11.38,2.19)(10.46,0.14)(12.16,-0.83)
\pspolygon[linecolor=qqzzff,fillcolor=qqzzff,fillstyle=solid,opacity=0.1](11.38,2.19)(13.54,1.02)(12.16,-0.83)
\pspolygon[linecolor=zzttqq,fillcolor=zzttqq,fillstyle=solid,opacity=0.1](5.02,5.56)(3.98,-0.32)(8.3,-5.5)(9.49,0.53)
\psline[linecolor=qqzzff](11.38,2.19)(10.46,0.14)
\psline[linecolor=qqzzff](10.46,0.14)(12.16,-0.83)
\psline[linecolor=afeeee](12.16,-0.83)(11.38,2.19)
\psline[linecolor=qqzzff](11.38,2.19)(13.54,1.02)
\psline[linecolor=qqzzff](13.54,1.02)(12.16,-0.83)
\psline[linecolor=qqzzff](12.16,-0.83)(11.38,2.19)
\psline[linestyle=dashed,dash=3pt 3pt,linecolor=qqzzff](10.46,0.14)(13.54,1.02)
\psline[linecolor=zzttqq](5.02,5.56)(3.98,-0.32)
\psline[linecolor=zzttqq](3.98,-0.32)(8.3,-5.5)
\psline[linecolor=zzttqq](8.3,-5.5)(9.49,0.53)
\psline[linecolor=zzttqq](9.49,0.53)(5.02,5.56)
\psline[linewidth=0.4pt,linestyle=dashed,dash=3pt 3pt,linecolor=blue](8.59,-2.02)(4.98,-0.44)
\psline[linewidth=0.4pt,linestyle=dashed,dash=3pt 3pt,linecolor=blue](4.98,-0.44)(5.74,3.48)
\psline[linewidth=0.4pt,linestyle=dashed,dash=3pt 3pt,linecolor=blue](5.74,3.48)(6.22,1.34)
\psline[linewidth=0.4pt,linestyle=dashed,dash=3pt 3pt,linecolor=blue](6.22,1.34)(8.59,-2.02)
\psline[linewidth=0.4pt,linestyle=dashed,dash=3pt 3pt,linecolor=blue](5.74,3.48)(8.59,-2.02)
\psline[linewidth=0.4pt,linestyle=dashed,dash=3pt 3pt,linecolor=blue](4.98,-0.44)(6.22,1.34)
\psline[linestyle=dotted](4.98,-0.44)(9.39,0.03)
\psline(7.15,3.16)(17.21,0.85)
\psline[linestyle=dotted](5.74,3.48)(7.15,3.16)
\psline(9.02,-1.88)(17.21,0.85)
\psline[linestyle=dotted](8.59,-2.02)(9.02,-1.88)
\psline[linestyle=dotted](6.22,1.34)(8.88,1.22)
\rput[tl](7.54,-3.49){$\pi$ }
\psline(13.54,1.02)(17.21,0.85)
\psline(13.1,0.42)(17.21,0.85)
\psline(9.39,0.03)(10.46,0.14)
\psline[linestyle=dotted,linecolor=blue](10.46,0.14)(13.1,0.42)
\psline(8.88,1.22)(11.47,1.11)
\psline[linestyle=dotted,linecolor=xdxdff](11.47,1.11)(13.54,1.02)
\psdots[dotstyle=*,linecolor=blue](11.38,2.19)
\rput[bl](11.5,2.38){{$X_1$}}
\psdots[dotstyle=*,linecolor=blue](10.46,0.14)
\rput[bl](10,0.54){{$X_2$}}
\psdots[dotstyle=*,linecolor=blue](12.16,-0.83)
\rput[bl](12.41,-1.52){{$X_3$}}
\psdots[dotstyle=*,linecolor=blue](13.54,1.02)
\rput[bl](13.6,1.1){{$X_4$}}
\psdots[dotstyle=*](17.21,0.85)
\rput[bl](17.34,1.04){$S$}
\psdots[dotstyle=*,linecolor=uququq](5.74,3.48)
\rput[bl](4.91,3.57){{$Y_1$}}
\psdots[dotstyle=*,linecolor=uququq](6.22,1.34)
\rput[bl](6.07,0.38){{$Y_4$}}
\psdots[dotstyle=*,linecolor=uququq](8.59,-2.02)
\rput[bl](8.23,-2.71){{$Y_3$}}
\psdots[dotstyle=*,linecolor=uququq](4.98,-0.44)
\rput[bl](4.29,-0.46){{$Y_2$}}
\psdots[dotsize=1pt 0,dotstyle=*,linecolor=bcduew](11.47,1.11)
\end{pspicture*}
\end{center}
\caption{Extended Desarguesian configuration for Central projection}
\label{fig:Desarguesian configuration projection}
\end{figure}

For example, if the homogeneous coordinates of the projection center $S$ and the projection hyperplane $\pi$ of a central projection $\mathscr{T}$ are:
\begin{equation*}
(s) = (0, 0, 0, 1)^{\scriptscriptstyle \top} \quad {\rm and} \quad
(\pi) = (0, 0, 1, -d)^{\scriptscriptstyle \top}
\end{equation*}
respectively, since a central projection matrix $\mathscr{T}$ should be singular, by using equation~\eqref{singular stereohomology equation} we obtain the central projection matrix as equation \eqref{central projection example}:

\begin{equation}\label{central projection example}
\begin{gathered}
           {\displaystyle \mathscr{T}\left((s),(\pi)\right)\mathop  {=\!\!=\!\!=} \limits^{def} {\boldsymbol I} - \frac{{(s)\!\cdot\!(\pi)\! ^\top }}
{{(s)\!^\top\!\!\cdot\! (\pi) }}} \\
  =\!\!\!\!=  {\scriptstyle \left[ {\begin{array}{*{20}c}
   1 & 0 & 0 & 0  \\
   0 & 1 & 0 & 0  \\
   0 & 0 & 1 & 0  \\
   0 & 0 & 0 & 1  \\

 \end{array} } \right]} -  {\scriptstyle \frac{{\left[ {\begin{array}{*{20}c}
   0  \\
   0  \\
   0  \\
   1  \\

 \end{array} } \right] \cdot \left[ {\begin{array}{*{20}c}
   0 & 0 & 1 & { - d}  \\

 \end{array} } \right]}}
{{\left[
{\begin{array}{*{20}c}
   0 & 0 & 0 & 1  \\

 \end{array} } \right] \cdot \left[ {\begin{array}{*{20}c}
   0  \\
   0  \\
   1  \\
   { - d}  \\

 \end{array} } \right]}} } =\!\!\!\!= \left[ {\begin{array}{*{20}c}
   1 & 0 & 0 & 0  \\
   0 & 1 & 0 & 0  \\
   0 & 0 & 1 & 0  \\
   0 & 0 &  d^{\text \,-1} & 0  \\

 \end{array} } \right] \hfill \\
\end{gathered}
\end{equation}

The result obtained in~\eqref{central projection example} can be verified by that in~\cite[pp.87~\texttildelow~92]{Salomon} unless row vector is used in~\cite{Salomon}. Note that no assumption is made on the reference coordinate system via the approach here, which means central projection matrix can be immediately obtained via equation~\eqref{singular stereohomology equation} provided arbitrary {\em projection center} and {\em projection hyperplane} are given.

\begin{definition}[Normal point]\label{normal point}
There is only one normal point in  ${\mathbb P}^3$, the homogeneous coordinate of which is $k\cdot(0, 0, 0, 1)$$^{\scriptscriptstyle \top}, 0 \neq k \in {\mathbb R}$.
\end{definition}

\begin{definition}[Ideal/Infinite hyperplane]\label{ideal hyperplane}
There is only one infinite (or ideal) hyperplane in ${\mathbb P}^3$, the homogeneous coordinate of which is $k\cdot(0, 0, 0, 1)$$^{\scriptscriptstyle \top}, 0 \neq k \in {\mathbb R}$.
\end{definition}

\begin{definition}[Direction]\label{direction}
A direction is a singular stereohomology, of which the stereohomology center $S$ is ordinary point and the stereohomology hyperplane $\pi$ is the infinite hyperplane. Figure~\ref{fig:Desarguesian configuration direction} illustrates the extended Desarguesian configuration of a typical direction, where the infinite geometric elements cannot be visualized immediately.

A direction with a stereohomology center $S$ is called a direction from $S$.
\end{definition}

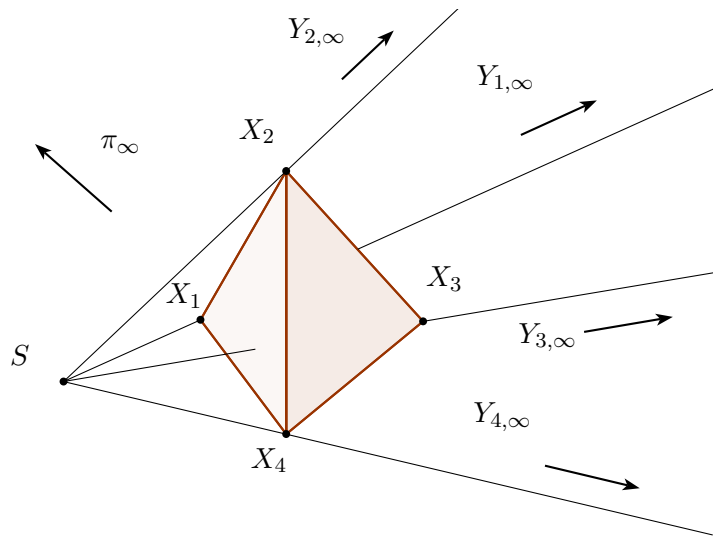
\begin{figure}[!hpt]
\begin{center}
\newrgbcolor{zzttqq}{0.6 0.2 0}
\newrgbcolor{xdxdff}{0.49 0.49 1}
\psset{xunit=1.0cm,yunit=1.0cm,algebraic=true,dimen=middle,dotstyle=o,dotsize=3pt 0,linewidth=0.8pt,arrowsize=3pt 2,arrowinset=0.25}
\begin{pspicture*}(4.42,-3.07)(14.41,4.03)
\pspolygon[linecolor=zzttqq,fillcolor=zzttqq,fillstyle=solid,opacity=0.04](7.6,-0.1)(8.74,-1.62)(8.74,1.88)
\pspolygon[linecolor=zzttqq,fillcolor=zzttqq,fillstyle=solid,opacity=0.1](8.74,-1.62)(10.56,-0.12)(8.74,1.88)
\psline[linecolor=zzttqq](7.6,-0.1)(8.74,-1.62)
\psline[linecolor=zzttqq](8.74,-1.62)(8.74,1.88)
\psline[linecolor=zzttqq](8.74,1.88)(7.6,-0.1)
\psline[linecolor=zzttqq](8.74,-1.62)(10.56,-0.12)
\psline[linecolor=zzttqq](10.56,-0.12)(8.74,1.88)
\psline[linecolor=zzttqq](8.74,1.88)(8.74,-1.62)
\psplot[linewidth=0.4pt]{5.78}{14.41}{(-18.91--2.8*x)/2.96}
\psplot[linewidth=0.4pt]{5.78}{14.41}{(--1.32-0.7*x)/2.96}
\psline[linewidth=0.4pt](5.78,-0.92)(7.6,-0.1)
\psline[linewidth=0.4pt](9.69,0.84)(22.36,6.55)
\psline[linewidth=0.4pt](10.56,-0.12)(25.9,2.45)
\psline[linewidth=0.4pt](5.78,-0.92)(8.33,-0.49)
\psline{->}(9.48,3.1)(10.17,3.75)
\psline{->}(11.86,2.36)(12.87,2.82)
\psline{->}(12.7,-0.26)(13.88,-0.06)
\psline{->}(12.18,-2.04)(13.44,-2.34)
\rput[tl](8.76,3.91){$Y_{2,\infty }$}
\rput[tl](11.27,3.29){$Y_{1,\infty }$}
\rput[tl](11.83,-0.14){$Y_{3,\infty }$}
\rput[tl](11.23,-1.2){$Y_{4,\infty }$}
\psline{->}(6.42,1.34)(5.4,2.24)
\rput[tl](6.27,2.35){$\pi_\infty $}
\psdots[dotstyle=*](7.6,-0.1)
\rput[bl](7.14,0.06){$X_1$}
\psdots[dotstyle=*](8.74,1.88)
\rput[bl](8.11,2.25){$X_2$}
\psdots[dotstyle=*](10.56,-0.12)
\rput[bl](10.6,0.29){$X_3$}
\psdots[dotstyle=*](8.74,-1.62)
\rput[bl](8.27,-2.13){$X_4$}
\psdots[dotstyle=*](5.78,-0.92)
\rput[bl](5.07,-0.7){$S$}
\psdots[dotstyle=*,linecolor=xdxdff](13.87,6.73)
\rput[bl](-2.26,6.52){\xdxdff{$F$}}
\psdots[dotstyle=*,linecolor=xdxdff](22.36,6.55)
\rput[bl](22.45,6.15){\xdxdff{$G$}}
\psdots[dotstyle=*,linecolor=xdxdff](25.9,2.45)
\rput[bl](-2.26,6.52){\xdxdff{$H$}}
\psdots[dotstyle=*,linecolor=xdxdff](22.49,-4.87)
\rput[bl](22.57,-4.73){\xdxdff{$I$}}
\end{pspicture*}
\end{center}
\caption{Extended Desarguesian configuration for direction}
\label{fig:Desarguesian configuration direction}
\end{figure}

\begin{definition}[Normal direction]\label{normal direction}
A normal direction is a direction from the normal point.

Since$\left( s \right) = {\left( {0,0,0,1} \right)^\top}$,$\left( \pi  \right) = {\left( {0,0,0,1} \right)^\top}$, and ${\left( s \right)^\top}\cdot\left( \pi  \right) = $1,  a normal direction thus obtained is:
\end{definition} \begin{equation}\label{normal direction equation}\begin{array}{l}
\mathscr{T} = \left[ {\begin{array}{*{20}{c}}
1&0&0&0\\
0&1&0&0\\
0&0&1&0\\
0&0&0&1
\end{array}} \right] - \left[ {\begin{array}{*{20}{c}}
0\\
0\\
0\\
1
\end{array}} \right]\cdot\left[ {\begin{array}{*{20}{c}}
0&0&0&1
\end{array}} \right] = \left[ {\begin{array}{*{20}{c}}
1&0&0&0\\
0&1&0&0\\
0&0&1&0\\
0&0&0&0
\end{array}} \right]
\end{array}\end{equation}
Such a definition as {\em normal direction} mainly serves as a bridge for the purpose of geometric meaning compatibility with Euclidean nomenclatures.

\begin{definition}[Parallel projection]\label{parallel projection}
A parallel projection is a singular stereohomology, of which the stereohomology center $S$ is an infinite point and the stereohomology hyperplane $\pi$ is an ordinary (finite) hyperplane. See figure~\ref{fig:Desarguesian configuration parallel projection}.

If $S$ is the normal direction of $\pi$, a parallel projection is orthographic, otherwise it is oblique.

See definition No. 2 in table \ref{classification table} for the parallel projection formula.
\end{definition}

\begin{figure}[!hpt]
\begin{center}
\newrgbcolor{zzttqq}{0.6 0.2 0}
\newrgbcolor{xdxdff}{0.49 0.49 1}
\newrgbcolor{qqzzff}{0 0.6 1}
\newrgbcolor{sqsqsq}{0.13 0.13 0.13}
\newrgbcolor{aqaqaq}{0.63 0.63 0.63}
\psset{xunit=0.68cm,yunit=0.68cm,algebraic=true,dimen=middle,dotstyle=o,dotsize=3pt 0,linewidth=0.8pt,arrowsize=3pt 2,arrowinset=0.25}
\begin{pspicture*}(-3.07,-3.2)(12.6,7.06)
\pspolygon[linecolor=zzttqq,fillcolor=zzttqq,fillstyle=solid,opacity=0.1](7.5,3.65)(5.57,2.1)(6.42,-1.2)(10.07,0.05)
\pspolygon[linecolor=qqzzff,fillcolor=qqzzff,fillstyle=solid,opacity=0.09](2.72,6.62)(-0.84,3.83)(1.39,-3.01)(4.48,0.61)
\pspolygon[linestyle=dashed,dash=2pt 2pt,linecolor=gray](1.42,4.76)(2.27,1.52)(1.55,-0.27)(3.42,2.42)
\pspolygon[linestyle=dashed,dash=2pt 2pt,linecolor=gray](1.55,-0.27)(1.42,4.76)(2.27,1.52)
\pspolygon[linecolor=zzttqq,fillcolor=zzttqq,fillstyle=solid,opacity=0.13](6.42,-1.2)(5.57,2.1)(7.5,3.65)
\psline[linecolor=zzttqq](7.5,3.65)(5.57,2.1)
\psline[linecolor=zzttqq](5.57,2.1)(6.42,-1.2)
\psline[linecolor=zzttqq](6.42,-1.2)(10.07,0.05)
\psline[linecolor=zzttqq](10.07,0.05)(7.5,3.65)
\psline[linecolor=xdxdff](6.42,-1.2)(7.5,3.65)
\psline[linecolor=qqzzff](2.72,6.62)(-0.84,3.83)
\psline[linecolor=qqzzff](-0.84,3.83)(1.39,-3.01)
\psline[linecolor=qqzzff](1.39,-3.01)(4.48,0.61)
\psline[linecolor=qqzzff](4.48,0.61)(2.72,6.62)
\psline[linestyle=dashed,dash=3pt 3pt,linecolor=gray](2.27,1.52)(1.55,-0.27)
\psline[linestyle=dashed,dash=2pt 2pt,linecolor=gray](1.55,-0.27)(3.42,2.42)
\psline[linestyle=dashed,dash=3pt 3pt,linecolor=gray](3.42,2.42)(1.42,4.76)
\psline[linecolor=xdxdff](1.42,4.76)(13.21,2.6)
\psline[linecolor=xdxdff](-3.51,5.66)(0.55,4.92)
\psline[linecolor=xdxdff](-3.48,3.76)(-0.65,3.24)
\psline[linecolor=xdxdff](-3.52,2.54)(-0.23,1.94)
\psline[linecolor=xdxdff](-3.33,0.59)(0.44,-0.1)
\psline[linecolor=xdxdff](1.55,-0.27)(13.52,-2.5)
\psline[linecolor=xdxdff](10.07,0.05)(13.33,-0.54)
\psline[linecolor=xdxdff](2.27,1.52)(6.32,0.74)
\psline[linecolor=xdxdff](3.42,2.42)(5.57,2.1)
\psline[linecolor=xdxdff](9.06,1.46)(13.36,0.68)
\psline[linestyle=dashed,dash=2pt 2pt,linecolor=gray](1.55,-0.27)(1.42,4.76)
\psline[linestyle=dashed,dash=3pt 3pt,linecolor=gray](1.42,4.76)(2.27,1.52)
\psline[linecolor=zzttqq](6.42,-1.2)(5.57,2.1)
\psline[linecolor=zzttqq](5.57,2.1)(7.5,3.65)
\psline[linecolor=zzttqq](7.5,3.65)(6.42,-1.2)
\rput[tl](10.24,2.77){$S_{\infty }$}
\rput[tl](-1.79,5.04){$S_{\infty }$}
\psline[linewidth=0.4pt]{->}(11.19,2)(12.39,1.78)
\psline[linewidth=0.4pt]{->}(-2.01,4.41)(-3,4.59)
\rput[tl](1.4,-1.54){$\pi$}
\psline[linestyle=dashed,dash=3pt 3pt,linecolor=aqaqaq](2.27,1.52)(3.42,2.42)
\psdots[dotstyle=*](7.5,3.65)
\rput[bl](7.64,3.84){$X_1$}
\psdots[dotstyle=*](5.57,2.1)
\rput[bl](5.25,2.56){$X_2$}
\psdots[dotstyle=*](6.42,-1.2)
\rput[bl](6.17,-2){$X_3$}
\psdots[dotstyle=*](10.07,0.05)
\rput[bl](10.21,0.32){$X_4$}
\psdots[dotstyle=*,linecolor=sqsqsq](1.42,4.76)
\rput[bl](1.55,4.95){\sqsqsq{$Y_1$}}
\psdots[dotstyle=*,linecolor=sqsqsq](2.27,1.52)
\rput[bl](2.28,2.44){\sqsqsq{$Y_4$}}
\psdots[dotstyle=*,linecolor=sqsqsq](1.55,-0.27)
\rput[bl](1.52,-0.99){\sqsqsq{$Y_3$}}
\psdots[dotstyle=*,linecolor=sqsqsq](3.42,2.42)
\rput[bl](3.54,2.59){\sqsqsq{$Y_2$}}
\psdots[dotsize=1pt 0,dotstyle=*,linecolor=lightgray](6.32,0.74)
\end{pspicture*}
\end{center}
\caption{Extended Desarguesian configuration for parallel projection}
\label{fig:Desarguesian configuration parallel projection}
\end{figure}
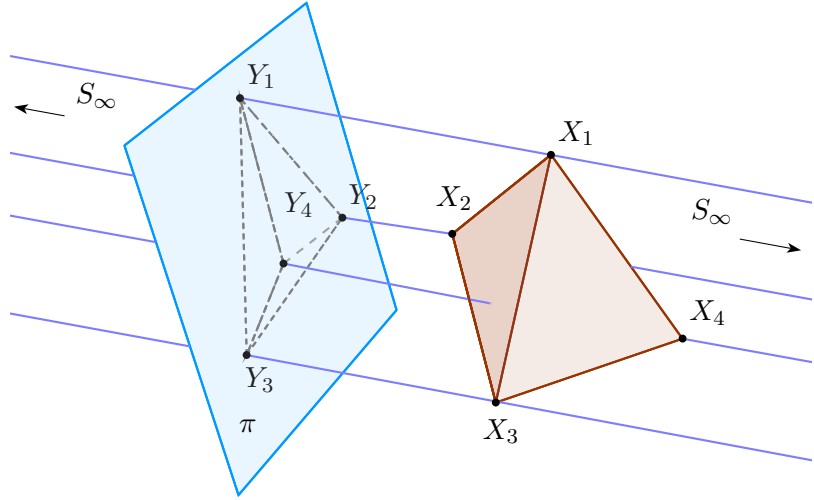

There is already similar representation on central and parallel projections in \cite[pp.67~\texttildelow~76]{CAD}, but no such general definition as stereohomology is found and the representation on reflections and rotations are still conventional ones in \cite{CAD}.

\begin{definition}[Space homology]\label{space homology}
A space homology is a nonsingular stereohomology, of which both the stereohomology center $S$ and the stereohomology hyperplane $\pi$ are ordinary (finite) elements and $S$ $\notin$ $\pi$. See figure~\ref{fig:Desarguesian configuration space homology}.

See definition No. 4 and No. 7 in table \ref{classification table} for the formula of space homology and its involutory case.
\end{definition}

\begin{figure}[!hpt]
\begin{center}
\newrgbcolor{zzttqq}{0.6 0.2 0}
\newrgbcolor{qqzzff}{0 0.6 1}
\newrgbcolor{xfqqff}{0.5 0 1}
\psset{xunit=0.045cm,yunit=0.045cm,algebraic=true,dimen=middle,dotstyle=o,dotsize=3pt 0,linewidth=0.8pt,arrowsize=3pt 2,arrowinset=0.25}
\begin{pspicture*}(-52.74,-39.55)(138.5,96.39)
\pspolygon[linewidth=0.4pt,linecolor=zzttqq,fillcolor=zzttqq,fillstyle=solid,opacity=0.1](27.65,-9)(43.55,3.3)(25.5,23.4)
\pspolygon[linewidth=0.4pt,linecolor=zzttqq,fillcolor=zzttqq,fillstyle=solid,opacity=0.1](99.03,-35.33)(125.02,9.97)(50.21,49.39)
\pspolygon[linewidth=0.4pt,linecolor=qqzzff,fillcolor=qqzzff,fillstyle=solid,opacity=0.05](-36.33,74.98)(9.34,81.76)(-2.43,-28.13)(-47.03,-32.77)
\psline[linewidth=4pt](-140.15,94.96)(-68.8,94.96)
\psline[linewidth=4pt](-140.15,86.4)(-68.8,86.4)
\psline[linewidth=4pt](-139.8,80.33)(-68.44,80.33)
\psline[linewidth=4pt](-140.15,71.77)(-68.8,71.77)
\psline[linewidth=4pt](-139.44,63.56)(-68.08,63.56)
\psline[linewidth=4pt](-139.08,55.36)(-67.72,55.36)
\psline[linewidth=4pt](-141.58,45.01)(-70.22,45.01)
\psline[linewidth=4pt](-142.29,29.67)(-70.94,29.67)
\psline[linewidth=4pt](-139.44,16.82)(-68.08,16.82)
\psline[linewidth=4pt](-140.51,3.62)(-69.15,3.62)
\psline[linewidth=4pt](-139.44,-7.44)(-68.08,-7.44)
\psline[linewidth=4pt](-139.08,-16)(-67.72,-16)
\psline[linewidth=4pt](-139.44,-27.06)(-68.08,-27.06)
\psline[linewidth=4pt](-134.8,-43.48)(-63.44,-43.48)
\psline[linewidth=4pt](-136.94,-58.1)(-65.58,-58.1)
\psline[linewidth=0.4pt,linecolor=zzttqq](35.6,23.3)(43.55,3.3)
\psline[linewidth=0.4pt,linecolor=zzttqq](125.02,9.97)(127.47,91.46)
\psline[linewidth=0.4pt,linecolor=zzttqq](27.65,-9)(43.55,3.3)
\psline[linewidth=0.4pt,linecolor=zzttqq](43.55,3.3)(25.5,23.4)
\psline[linewidth=0.4pt,linecolor=zzttqq](25.5,23.4)(27.65,-9)
\psline[linewidth=0.4pt,linecolor=zzttqq](99.03,-35.33)(125.02,9.97)
\psline[linewidth=0.4pt,linecolor=zzttqq](125.02,9.97)(50.21,49.39)
\psline[linewidth=0.4pt,linecolor=zzttqq](50.21,49.39)(99.03,-35.33)
\psline[linewidth=0.4pt,linecolor=qqzzff](-36.33,74.98)(9.34,81.76)
\psline[linewidth=0.4pt,linecolor=qqzzff](9.34,81.76)(-2.43,-28.13)
\psline[linewidth=0.4pt,linecolor=qqzzff](-2.43,-28.13)(-47.03,-32.77)
\psline[linewidth=0.4pt,linecolor=qqzzff](-47.03,-32.77)(-36.33,74.98)
\psline[linewidth=0.4pt,linecolor=xfqqff](3.25,0)(125.02,9.97)
\psline[linewidth=0.4pt,linecolor=xfqqff](3.25,0)(99.03,-35.33)
\psline[linewidth=0.4pt,linecolor=xfqqff](3.25,0)(50.21,49.39)
\rput[tl](-31.33,54.64){$\pi$}
\psline[linewidth=0.4pt,linecolor=xfqqff](3.25,0)(25.94,16.71)
\psline[linewidth=0.4pt,linecolor=xfqqff](35.6,23.3)(56.2,38.99)
\psline[linewidth=0.4pt,linestyle=dashed,dash=1pt 1pt,linecolor=xfqqff](56.2,38.99)(127.47,91.46)
\psline[linewidth=0.4pt,linecolor=zzttqq](25.5,23.4)(35.6,23.3)
\psline[linewidth=0.4pt,linecolor=zzttqq](50.21,49.39)(127.47,91.46)
\psline[linewidth=0.4pt,linestyle=dashed,dash=1pt 1pt,linecolor=zzttqq](99.03,-35.33)(127.47,91.46)
\psline[linewidth=0.4pt,linestyle=dashed,dash=1pt 1pt,linecolor=zzttqq](27.65,-9)(35.6,23.3)
\psline[linewidth=0.4pt,linestyle=dashed,dash=1pt 1pt,linecolor=xfqqff](25.94,16.71)(35.6,23.3)
\begin{scriptsize}
\psdots[dotsize=5pt 0,dotstyle=*](-134.45,94.96)
\rput[bl](-138.37,99.24){$a_3$}
\psdots[dotsize=5pt 0,dotstyle=*](-99.84,86.4)
\rput[bl](-103.76,91.75){$a1_2$}
\psdots[dotsize=5pt 0,dotstyle=*](-128,80.33)
\rput[bl](-131.95,85.68){$a2_2$}
\psdots[dotsize=5pt 0,dotstyle=*](-83.45,71.77)
\rput[bl](-87.35,77.12){$a3_2$}
\psdots[dotsize=5pt 0,dotstyle=*](-103.36,63.56)
\rput[bl](-107.33,68.91){$b1_2$}
\psdots[dotsize=5pt 0,dotstyle=*](-96.65,55.36)
\rput[bl](-100.55,60.71){$b2_2$}
\psdots[dotsize=5pt 0,dotstyle=*](-100.52,45.01)
\rput[bl](-104.47,50.36){$b3_2$}
\psdots[dotsize=5pt 0,dotstyle=*](-99.73,29.67)
\rput[bl](-103.76,35.02){$c1_2$}
\psdots[dotsize=5pt 0,dotstyle=*](-125.47,16.82)
\rput[bl](-129.45,22.17){$c2_2$}
\psdots[dotsize=5pt 0,dotstyle=*](-111.25,3.62)
\rput[bl](-115.18,8.97){$c3_2$}
\psdots[dotsize=5pt 0,dotstyle=*](-96.98,-7.44)
\rput[bl](-100.91,-2.09){$d1_2$}
\psdots[dotsize=5pt 0,dotstyle=*](-81.17,-16)
\rput[bl](-85.21,-10.65){$d2_2$}
\psdots[dotsize=5pt 0,dotstyle=*](-108.47,-27.06)
\rput[bl](-112.32,-21.71){$d3_2$}
\psdots[dotsize=5pt 0,dotstyle=*](-104.83,-43.48)
\rput[bl](-108.76,-38.12){$b_3$}
\psdots[dotstyle=*](35.6,23.3)
\rput[bl](43.24,21.1){$X_1$}
\psdots[dotstyle=*](43.55,3.3)
\rput[bl](45.38,-2.8){$X_2$}
\psdots[dotstyle=*](27.65,-9)
\rput[bl](22.9,-16.72){$X_3$}
\psdots[dotstyle=*](25.5,23.4)
\rput[bl](17.91,26.81){$X_4$}
\psdots[dotstyle=*](3.25,0)
\rput[bl](1.85,-7.44){$S$}
\psdots[dotsize=5pt 0,dotstyle=*](-103.4,-58.1)
\rput[bl](-107.33,-52.75){$\sigma = -0.3$}
\psdots[dotstyle=*](127.47,91.46)
\rput[bl](129.94,81.4){$Y_1$}
\psdots[dotstyle=*](125.02,9.97)
\rput[bl](124.94,1.48){$Y_2$}
\psdots[dotstyle=*](99.03,-35.33)
\rput[bl](106.03,-35.98){$Y_3$}
\psdots[dotstyle=*](50.21,49.39)
\rput[bl](44.67,53.57){$Y_4$}
\end{scriptsize}
\end{pspicture*}
\end{center}
\caption{Extended Desarguesian configuration for space homology}
\label{fig:Desarguesian configuration space homology}
\end{figure}
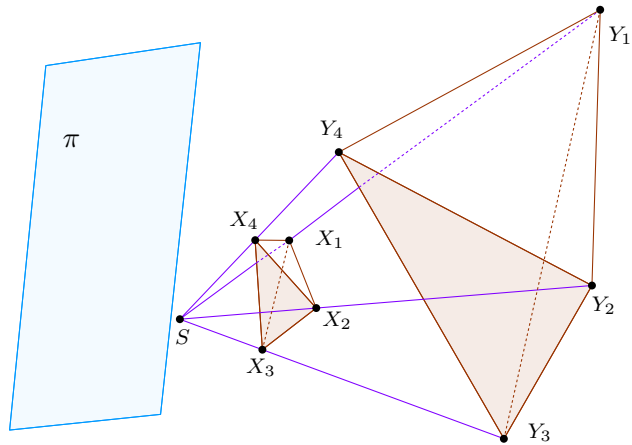

\begin{definition}[Elementary scaling]\label{scaling}
An elementary scaling is a nonsingular stereohomology, of which the stereohomology center $S$ is infinite while the stereohomology hyperplane $\pi$ is ordinary (finite), and $S$ $\notin$ $\pi$.  $S$ is called the scaling direction. See figure~\ref{fig:Desarguesian configuration scaling}.

If $S$ is the normal direction of $\pi$, a scaling is orthographic, otherwise it is oblique.  See definition No. 5 in table \ref{classification table} for the scaling formula, where the the scaling ratio is: ${\rho}$\slash${\lambda}$.
\end{definition}

\begin{figure}[!hpt]
\begin{center}
\newrgbcolor{zzttqq}{0.6 0.2 0}
\newrgbcolor{xfqqff}{0.5 0 1}
\psset{xunit=0.5cm,yunit=0.5cm,algebraic=true,dimen=middle,dotstyle=o,dotsize=3pt 0,linewidth=0.8pt,arrowsize=3pt 2,arrowinset=0.25}
\begin{pspicture*}(-6.47,-3.14)(21.96,14.12)
\pspolygon[linecolor=zzttqq,fillcolor=zzttqq,fillstyle=solid,opacity=0.04](2.98,2.35)(6.1,1.51)(2.02,4.98)
\pspolygon[linecolor=zzttqq,fillcolor=zzttqq,fillstyle=solid,opacity=0.1](2.02,4.98)(7.22,3.57)(6.1,1.51)
\pspolygon[linecolor=zzttqq,fillcolor=zzttqq,fillstyle=solid,opacity=0.05](9.06,3.79)(11.77,12.51)(7.57,5.89)
\pspolygon[linecolor=zzttqq,fillcolor=zzttqq,fillstyle=solid,opacity=0.1](9.06,3.79)(14.21,8.96)(11.77,12.51)
\psplot{-6.47}{21.96}{(-0-0*x)/1}
\psline[linecolor=zzttqq](2.98,2.35)(6.1,1.51)
\psline[linecolor=zzttqq](6.1,1.51)(2.02,4.98)
\psline[linecolor=zzttqq](2.02,4.98)(2.98,2.35)
\psline[linecolor=zzttqq](2.02,4.98)(7.22,3.57)
\psline[linecolor=zzttqq](7.22,3.57)(6.1,1.51)
\psline[linecolor=zzttqq](6.1,1.51)(2.02,4.98)
\psline[linecolor=zzttqq](9.06,3.79)(11.77,12.51)
\psline[linecolor=zzttqq](11.77,12.51)(7.57,5.89)
\psline[linecolor=zzttqq](7.57,5.89)(9.06,3.79)
\psline[linecolor=zzttqq](9.06,3.79)(14.21,8.96)
\psline[linecolor=zzttqq](14.21,8.96)(11.77,12.51)
\psline[linecolor=zzttqq](11.77,12.51)(9.06,3.79)
\psline[linewidth=0.4pt,linecolor=xfqqff](-4.77,-3.63)(2.98,2.35)
\psline[linewidth=0.4pt,linecolor=xfqqff](13.29,10.3)(21.2,16.4)
\psline[linewidth=0.4pt,linecolor=xfqqff](14.21,8.96)(23.58,16.19)
\psline[linewidth=0.4pt,linecolor=xfqqff](-2.29,-3.78)(5.17,1.98)
\psline[linewidth=0.4pt,linecolor=xfqqff](7.22,3.57)(8.99,4.93)
\psline[linewidth=0.4pt,linecolor=xfqqff](-0.69,-3.73)(25.16,16.22)
\psline[linewidth=0.4pt,linecolor=xfqqff](-8.77,-3.34)(16.84,16.42)
\psline[linewidth=0.4pt,linestyle=dashed,dash=2pt 2pt,linecolor=zzttqq](20.36,0)(14.21,8.96)
\psline[linewidth=0.4pt,linestyle=dashed,dash=2pt 2pt,linecolor=zzttqq](7.22,3.57)(20.36,0)
\psline[linewidth=0.4pt,linestyle=dashed,dash=2pt 2pt,linecolor=zzttqq](11.73,0)(9.06,3.79)
\psline[linewidth=0.4pt,linestyle=dashed,dash=2pt 2pt,linecolor=zzttqq](6.1,1.51)(11.73,0)
\psline[linewidth=0.4pt,linestyle=dashed,dash=2pt 2pt,linecolor=zzttqq](7.87,0)(9.06,3.79)
\psline[linewidth=0.4pt,linestyle=dashed,dash=2pt 2pt,linecolor=zzttqq](7.87,0)(6.1,1.51)
\psline[linewidth=0.4pt,linestyle=dashed,dash=2pt 2pt,linecolor=zzttqq](5.27,0)(6.1,1.51)
\psline[linewidth=0.4pt,linestyle=dashed,dash=2pt 2pt,linecolor=zzttqq](3.84,0)(2.98,2.35)
\psline[linewidth=0.4pt,linestyle=dashed,dash=2pt 2pt,linecolor=zzttqq](-5.17,0)(2.54,3.56)
\psline[linewidth=0.4pt,linestyle=dashed,dash=2pt 2pt,linecolor=zzttqq](-5.17,0)(2.98,2.35)
\psline[linewidth=0.4pt,linestyle=dashed,dash=2pt 2pt,linecolor=zzttqq](4.28,4.37)(7.57,5.89)
\rput[tl](10.13,2.99){$\text{The crosssection line of }\;\pi$}
\psline[linewidth=0.4pt]{->}(14.28,1.96)(14.98,0)
\psline[linewidth=0.4pt]{->}(14.32,12.11)(15.84,13.28)
\psline[linewidth=0.4pt]{->}(-2.86,-1.15)(-4.54,-2.44)
\psline[linewidth=0.4pt,linestyle=dashed,dash=2pt 2pt,linecolor=zzttqq](6.27,3.83)(7.57,5.89)
\psline[linewidth=0.4pt,linestyle=dashed,dash=2pt 2pt,linecolor=zzttqq](4.99,1.81)(3.84,0)
\psline[linewidth=0.4pt,linestyle=dashed,dash=2pt 2pt,linecolor=zzttqq](5.27,0)(9.06,3.79)
\rput[tl](13.13,13.47){$S_\infty $}
\rput[tl](-4.87,-0.26){$S_\infty $}
\psline[linewidth=0.4pt,linecolor=xfqqff](5.26,4.1)(7.57,5.89)
\psdots[dotstyle=*](2.98,2.35)
\rput[bl](1.58,2.45){$X_2$}
\psdots[dotstyle=*](6.1,1.51)
\rput[bl](5.4,2.25){$X_1$}
\psdots[dotstyle=*](7.22,3.57)
\rput[bl](6.92,4.06){$X_4$}
\psdots[dotstyle=*](2.02,4.98)
\rput[bl](1.13,5.21){$X_3$}
\psdots[dotsize=1pt 0,dotstyle=*](3.84,0)
\rput[bl](2.85,-1.49){$S_{23}$}
\psdots[dotsize=1pt 0,dotstyle=*](5.27,0)
\rput[bl](4.91,-1.45){$S_{14}$}
\psdots[dotsize=1pt 0,dotstyle=*](7.87,0)
\rput[bl](7.25,-1.49){$S_{13}$}
\psdots[dotsize=1pt 0,dotstyle=*](11.73,0)
\rput[bl](11.11,-1.29){$S_{12}$}
\psdots[dotsize=1pt 0,dotstyle=*](20.36,0)
\rput[bl](19.42,-1.2){$S_{34}$}
\psdots[dotstyle=*](9.06,3.79)
\rput[bl](9.55,3.52){$Y_1$}
\psdots[dotstyle=*](7.57,5.89)
\rput[bl](6.51,6.52){$Y_2$}
\psdots[dotstyle=*](14.21,8.96)
\rput[bl](12.84,8.95){$Y_4$}
\psdots[dotstyle=*](11.77,12.51)
\rput[bl](10.25,12.52){$Y_3$}
\psdots[dotstyle=*](-5.17,0)
\rput[bl](-4.95,0.93){$S_{24}$}
\end{pspicture*}
\end{center}
\caption{Extended Desarguesian configuration for scaling}
\label{fig:Desarguesian configuration scaling}
\end{figure}
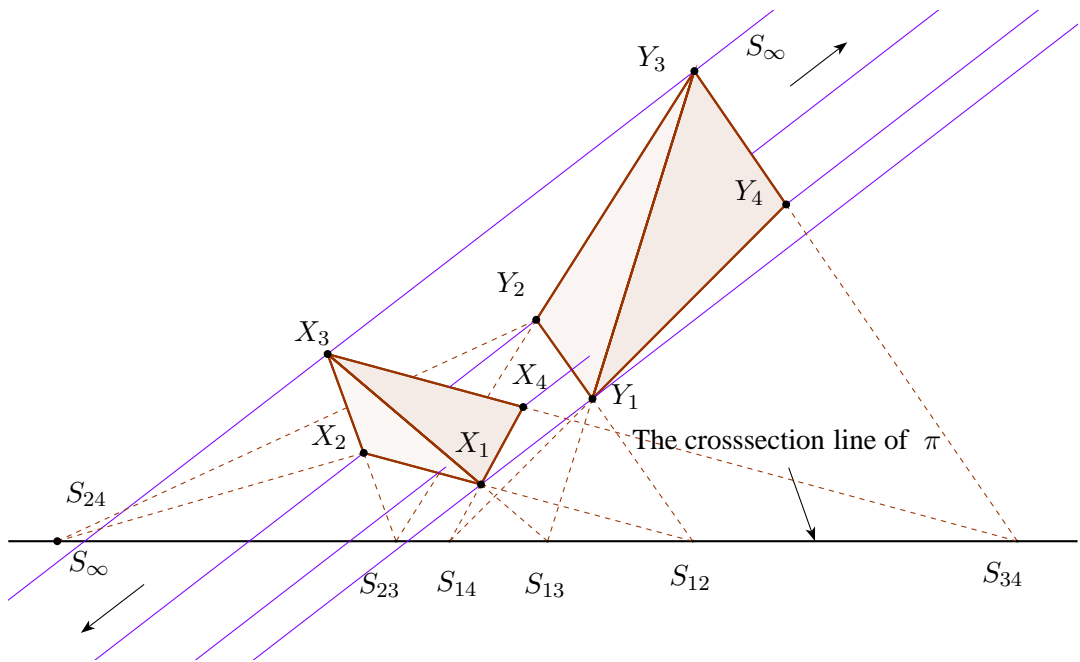

A general scaling with different scaling ratios in multiple directions is not elementary, which can be obtained by concatenating elementary scaling matrices together.

\begin{definition}[Central dilation]\label{central dilation}
A central dilation is a nonsingular stereohomology, of which the stereohomology center $S$ is ordinary while the stereohomology hyperplane $\pi$ is infinite), and $S$ $\notin$ $\pi$. See figure~\ref{fig:Desarguesian configuration central dilation}.

See definition No. 6 in table \ref{classification table} for the central dilation formula. The central dilation ratio is ${\rho}$\slash${\lambda}$.
\end{definition}

Note that conventional representation usually has the difficulty in distinguishing {\em central dilation} from {\em scaling} due to its ambiguity in definition.

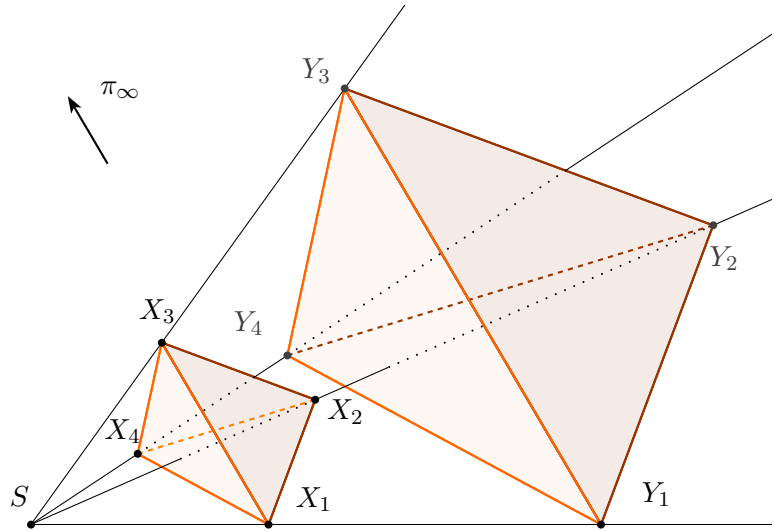
\begin{figure}[!hpt]
\begin{center}
\newrgbcolor{zzttqq}{0.6 0.2 0}
\newrgbcolor{ffxfqq}{1 0.5 0}
\newrgbcolor{ffwwqq}{1 0.4 0}
\psset{xunit=1.0cm,yunit=1.0cm,algebraic=true,dimen=middle,dotstyle=o,dotsize=3pt 0,linewidth=0.8pt,arrowsize=3pt 2,arrowinset=0.25}
\begin{pspicture*}(4.54,0.64)(15.1,8.04)
\pspolygon[linecolor=zzttqq,fillcolor=zzttqq,fillstyle=solid,opacity=0.1](7,3.56)(8.42,1.14)(9.04,2.8)
\pspolygon[linecolor=zzttqq,fillcolor=zzttqq,fillstyle=solid,opacity=0.1](12.84,1.14)(9.43,6.94)(14.33,5.12)
\pspolygon[linecolor=ffwwqq,fillcolor=ffwwqq,fillstyle=solid,opacity=0.05](8.67,3.39)(9.43,6.94)(12.84,1.14)
\pspolygon[linecolor=ffwwqq,fillcolor=ffwwqq,fillstyle=solid,opacity=0.05](6.68,2.08)(8.42,1.14)(7,3.56)
\psline[linecolor=zzttqq](7,3.56)(8.42,1.14)
\psline[linecolor=zzttqq](8.42,1.14)(9.04,2.8)
\psline[linecolor=zzttqq](9.04,2.8)(7,3.56)
\psline[linecolor=zzttqq](12.84,1.14)(9.43,6.94)
\psline[linecolor=zzttqq](9.43,6.94)(14.33,5.12)
\psline[linecolor=zzttqq](14.33,5.12)(12.84,1.14)
\psline[linestyle=dashed,dash=2pt 2pt,linecolor=ffxfqq](6.68,2.08)(9.04,2.8)
\psline[linestyle=dashed,dash=2pt 2pt,linecolor=zzttqq](8.67,3.39)(14.33,5.12)
\psline[linecolor=ffwwqq](8.67,3.39)(9.43,6.94)
\psline[linecolor=ffwwqq](9.43,6.94)(12.84,1.14)
\psline[linecolor=ffwwqq](12.84,1.14)(8.67,3.39)
\psline[linecolor=ffwwqq](6.68,2.08)(8.42,1.14)
\psline[linecolor=ffwwqq](8.42,1.14)(7,3.56)
\psline[linecolor=ffwwqq](7,3.56)(6.68,2.08)
\psline[linewidth=0.4pt](5.26,1.14)(10.59,8.55)
\psline[linewidth=0.4pt](5.26,1.14)(7.24,2.01)
\psline[linewidth=0.4pt](9.04,2.8)(10.01,3.23)
\psline[linewidth=0.4pt](14.33,5.12)(17.14,6.36)
\psline[linewidth=0.4pt](5.26,1.14)(6.68,2.08)
\psline[linewidth=0.4pt](12.37,5.85)(15.29,7.78)
\psline[linewidth=0.4pt](8.23,3.1)(8.67,3.39)
\psline[linewidth=0.4pt](5.26,1.14)(17.08,1.14)
\psline[linestyle=dotted](10.01,3.23)(14.33,5.12)
\psline[linestyle=dotted](8.67,3.39)(12.37,5.85)
\psline[linestyle=dotted](6.68,2.08)(8.23,3.1)
\psline[linestyle=dotted](7.24,2.01)(9.04,2.8)
\psline{->}(6.28,5.94)(5.75,6.84)
\rput[tl](6.18,7.04){$\pi _\infty $}
\psdots[dotstyle=*](5.26,1.14)
\rput[bl](4.98,1.38){$S$}
\psdots[dotstyle=*](7,3.56)
\rput[bl](6.7,3.82){$X_3$}
\psdots[dotstyle=*](8.42,1.14)
\rput[bl](8.78,1.32){$X_1$}
\psdots[dotstyle=*](9.04,2.8)
\rput[bl](9.2,2.48){$X_2$}
\psdots[dotstyle=*](12.84,1.14)
\rput[bl](13.38,1.38){$Y_1$}
\psdots[dotstyle=*,linecolor=darkgray](9.43,6.94)
\rput[bl](8.86,7.02){\darkgray{$Y_3$}}
\psdots[dotstyle=*,linecolor=darkgray](14.33,5.12)
\rput[bl](14.28,4.52){\darkgray{$Y_2$}}
\psdots[dotstyle=*](6.68,2.08)
\rput[bl](6.24,2.22){$X_4$}
\psdots[dotstyle=*,linecolor=darkgray](8.67,3.39)
\rput[bl](7.96,3.66){\darkgray{$Y_4$}}
\end{pspicture*}
\end{center}
\caption{Extended Desarguesian configuration for central dilation}
\label{fig:Desarguesian configuration central dilation}
\end{figure}

\begin{definition}[Reflection]\label{reflection}
A reflection is an involutory elementary scaling with scaling ratio $-$1, of which the stereohomology hyperplane $\pi$ is the reflection hyperplane or the mirror hyperplane. See figure~\ref{fig:Desarguesian configuration reflection}.

If the stereohomology center $S$ is the normal direction of $\pi$, a reflection is orthographic, otherwise skew or oblique.  Usually when we mention a reflection we mean orthographic reflection unless otherwise specified.

See definition No. 8 in table \ref{classification table} for the reflection formula.
\end{definition}

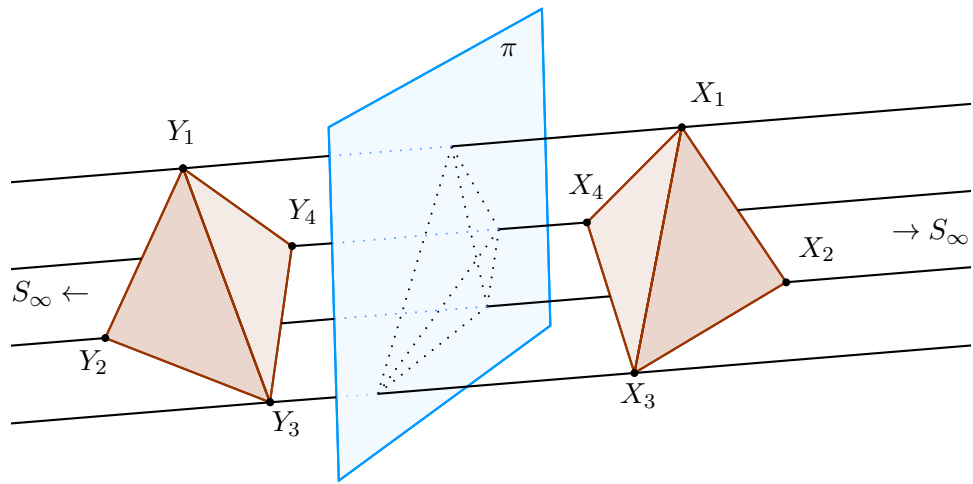
\begin{figure}[!hpt]
\begin{center}
\newrgbcolor{zzttqq}{0.6 0.2 0}
\newrgbcolor{qqzzff}{0 0.6 1}
\newrgbcolor{wwzzff}{0.4 0.6 1}
\newrgbcolor{zzccff}{0.6 0.8 1}
\psset{xunit=0.75cm,yunit=0.75cm,algebraic=true,dotstyle=o,dotsize=3pt 0,linewidth=0.8pt,arrowsize=3pt 2,arrowinset=0.25}
\begin{pspicture*}(-3,-4)(14,6)
\pspolygon[linecolor=zzttqq,fillcolor=zzttqq,fillstyle=solid,opacity=0.2](8.88,3.19)(10.73,0.44)(8.04,-1.17)
\pspolygon[linecolor=zzttqq,fillcolor=zzttqq,fillstyle=solid,opacity=0.1](8.88,3.19)(7.2,1.5)(8.04,-1.17)
\pspolygon[linecolor=zzttqq,fillcolor=zzttqq,fillstyle=solid,opacity=0.2](0.04,2.46)(-1.34,-0.55)(1.58,-1.69)
\pspolygon[linecolor=zzttqq,fillcolor=zzttqq,fillstyle=solid,opacity=0.1](0.04,2.46)(1.97,1.08)(1.58,-1.69)
\pspolygon[linecolor=qqzzff,fillcolor=qqzzff,fillstyle=solid,opacity=0.05](2.62,3.19)(6.4,5.29)(6.55,-0.33)(2.8,-3.08)
\psline[linecolor=zzttqq](8.88,3.19)(10.73,0.44)
\psline[linecolor=zzttqq](10.73,0.44)(8.04,-1.17)
\psline[linecolor=zzttqq](8.04,-1.17)(8.88,3.19)
\psline[linecolor=zzttqq](8.88,3.19)(7.2,1.5)
\psline[linecolor=zzttqq](7.2,1.5)(8.04,-1.17)
\psline[linecolor=zzttqq](8.04,-1.17)(8.88,3.19)
\psline[linecolor=zzttqq](0.04,2.46)(-1.34,-0.55)
\psline[linecolor=zzttqq](-1.34,-0.55)(1.58,-1.69)
\psline[linecolor=zzttqq](1.58,-1.69)(0.04,2.46)
\psline[linecolor=zzttqq](0.04,2.46)(1.97,1.08)
\psline[linecolor=zzttqq](1.97,1.08)(1.58,-1.69)
\psline[linecolor=zzttqq](1.58,-1.69)(0.04,2.46)
\psline[linecolor=qqzzff](2.62,3.19)(6.4,5.29)
\psline[linecolor=qqzzff](6.4,5.29)(6.55,-0.33)
\psline[linecolor=qqzzff](6.55,-0.33)(2.8,-3.08)
\psline[linecolor=qqzzff](2.8,-3.08)(2.62,3.19)
\psline[linestyle=dotted](4.8,2.85)(5.64,1.38)
\psline[linestyle=dotted](5.64,1.38)(5.43,0.01)
\psline[linestyle=dotted](3.49,-1.54)(5.43,0.01)
\psline[linestyle=dotted](4.8,2.85)(3.49,-1.54)
\psline[linestyle=dotted](4.8,2.85)(5.43,0.01)
\psline[linestyle=dotted](5.64,1.38)(3.49,-1.54)
\psline(5.64,1.38)(7.2,1.5)
\psline(4.8,2.85)(8.88,3.19)
\psline(3.49,-1.54)(8.04,-1.17)
\psline(5.43,0.01)(7.61,0.19)
\psline(1.58,-1.69)(2.76,-1.6)
\psline(1.78,-0.29)(2.72,-0.21)
\psline(1.97,1.08)(2.68,1.13)
\psline(0.04,2.46)(2.63,2.68)
\psline(-6.01,-2.31)(1.58,-1.69)
\psline(-6.16,-0.94)(-1.34,-0.55)
\psline(-6.28,1.95)(0.04,2.46)
\psline(-6.36,0.4)(-0.7,0.86)
\psline(8.88,3.19)(18.18,3.94)
\psline(8.04,-1.17)(18.67,-0.3)
\psline(10.73,0.44)(18.6,1.08)
\psline(9.87,1.72)(18.42,2.42)
\psline[linestyle=dotted,linecolor=wwzzff](2.63,2.68)(4.8,2.85)
\psline[linestyle=dotted,linecolor=wwzzff](2.68,1.13)(5.64,1.38)
\psline[linestyle=dotted,linecolor=wwzzff](2.72,-0.21)(5.43,0.01)
\psline[linestyle=dotted,linecolor=zzccff](2.76,-1.6)(3.49,-1.54)
\rput[tl](5.66,4.67){$\pi$}
\rput[tl](12.6,1.55){$\to S_{\infty}$}
\rput[tl](-3,0.45){$ S_{\infty} \leftarrow$}
\psdots[dotstyle=*](8.88,3.19)
\rput[bl](9.03,3.56){$X_1$}
\psdots[dotstyle=*](10.73,0.44)
\rput[bl](10.95,0.81){$X_2$}
\psdots[dotstyle=*](8.04,-1.17)
\rput[bl](7.79,-1.79){$X_3$}
\psdots[dotstyle=*](7.2,1.5)
\rput[bl](6.87,1.92){$X_4$}
\psdots[dotstyle=*](0.04,2.46)
\rput[bl](-0.25,2.89){$Y_1$}
\psdots[dotstyle=*](-1.34,-0.55)
\rput[bl](-1.81,-1.19){$Y_2$}
\psdots[dotstyle=*](1.58,-1.69)
\rput[bl](1.61,-2.28){$Y_3$}
\psdots[dotstyle=*](0.04,2.46)
\psdots[dotstyle=*](1.97,1.08)
\rput[bl](1.9,1.48){$Y_4$}
\psdots[dotstyle=*](1.58,-1.69)
\end{pspicture*}
\end{center}
\caption{Extended Desarguesian configuration for reflection}
\label{fig:Desarguesian configuration reflection}
\end{figure}

An {\em orthographic} reflection solution to Example 3.5 in \cite[pp.48~\texttildelow~49]{CAD} based on the definition here will be as follows:

Since the mirror hyperplane specified is $2x-y+2z-2=0$, then reflection mirror ($\pi$) = ($2,-1,2,-2$)$^\top$, and $S$ is the normal direction of $\pi$ by using equation \eqref{normal direction equation}:

\[\left( s \right) = \left[ {\begin{array}{*{20}{c}}
1&0&0&0\\
0&1&0&0\\
0&0&1&0\\
0&0&0&0
\end{array}} \right] \cdot \left[ {\begin{array}{*{20}{c}}
2\\
{ - 1}\\
2\\
{ - 2}
\end{array}} \right] = \left[ {\begin{array}{*{20}{c}}
2\\
{ - 1}\\
2\\
0
\end{array}} \right]\]

The reflection $\mathscr{R}$ thus obtained in equation \eqref{reflection example} via equation \eqref{involutory stereohomology equation} is the same as that in  \cite[p.49]{CAD} unless row vectors are used in \cite{CAD} while the approach used here is much simpler and can be extended to general cases without difficulty.
\begin{equation}\label{reflection example}\mathscr{R} = \left[ {\begin{array}{*{20}{c}}
1&0&0&0\\
0&1&0&0\\
0&0&1&0\\
0&0&0&1
\end{array}} \right] - 2 \cdot \frac{{\left[ {\begin{array}{*{20}{c}}
2\\[-12pt]
{ - 1}\\[-12pt]
2\\[-12pt]
0
\end{array}} \right] \cdot \left[ {\begin{array}{*{20}{c}}
2&{ - 1}&2&{ - 2}
\end{array}} \right]}}{{\left[ {\begin{array}{*{20}{c}}
2&{ - 1}&2&0
\end{array}} \right] \cdot \left[ {\begin{array}{*{20}{c}}
2\\[-12pt]
{ - 1}\\[-12pt]
2\\[-12pt]
{ - 2}
\end{array}} \right]}} = \frac{1}{9} \cdot \left[ {\begin{array}{*{20}{c}}
1&4&{ - 8}&8\\
4&7&4&{ - 4}\\
{ - 8}&4&1&8\\
0&0&0&9
\end{array}} \right]\end{equation}
\begin{definition}[Central symmetry]\label{central symmetry}
A central symmetry is an involutory central dilation with central dilation ratio $-$1, of which the stereohomology center $S$ is the central symmetry center.

See figure~\ref{fig:Desarguesian configuration central symmetry} for illustration and definition No. 9 in table \ref{classification table} for the central symmetry formulation.
\end{definition}

\begin{figure}[!hpt]
\begin{center}
\newrgbcolor{zzttqq}{0.6 0.2 0}
\psset{xunit=1.0cm,yunit=1.0cm,algebraic=true,dimen=middle,dotstyle=o,dotsize=3pt 0,linewidth=0.8pt,arrowsize=3pt 2,arrowinset=0.25}
\begin{pspicture*}(3.97,-2.3)(16.71,3.95)
\pspolygon[linecolor=zzttqq,fillcolor=zzttqq,fillstyle=solid,opacity=0.04](5.9,2.18)(4.78,0.32)(8.42,2.68)
\pspolygon[linecolor=zzttqq,fillcolor=zzttqq,fillstyle=solid,opacity=0.1](4.78,0.32)(8.44,-1.32)(8.42,2.68)
\pspolygon[linecolor=zzttqq,fillcolor=zzttqq,fillstyle=solid,opacity=0.1](11.76,2.92)(14.3,-0.58)(15.42,1.28)
\pspolygon[linecolor=zzttqq,fillcolor=zzttqq,fillstyle=solid,opacity=0.05](11.76,2.92)(11.78,-1.08)(14.3,-0.58)
\psline[linecolor=zzttqq](5.9,2.18)(4.78,0.32)
\psline[linecolor=zzttqq](4.78,0.32)(8.42,2.68)
\psline[linecolor=zzttqq](8.42,2.68)(5.9,2.18)
\psline[linecolor=zzttqq](4.78,0.32)(8.44,-1.32)
\psline[linecolor=zzttqq](8.44,-1.32)(8.42,2.68)
\psline[linecolor=zzttqq](8.42,2.68)(4.78,0.32)
\psplot[linewidth=0.4pt]{3.97}{16.71}{(-20.08--2.12*x)/1.66}
\psplot[linewidth=0.4pt]{3.97}{16.71}{(--20.33-1.88*x)/1.68}
\psline[linewidth=0.4pt,linestyle=dashed,dash=1pt 1pt](5.9,2.18)(8.44,-1.32)
\psline[linestyle=dashed,dash=1pt 1pt,linecolor=zzttqq](11.76,2.92)(14.3,-0.58)
\psline[linecolor=zzttqq](14.3,-0.58)(15.42,1.28)
\psline[linecolor=zzttqq](15.42,1.28)(11.76,2.92)
\psline[linecolor=zzttqq](11.76,2.92)(11.78,-1.08)
\psline[linecolor=zzttqq](11.78,-1.08)(14.3,-0.58)
\psline[linestyle=dashed,dash=1pt 1pt,linecolor=zzttqq](14.3,-0.58)(11.76,2.92)
\psline[linewidth=0.4pt,linestyle=dashed,dash=1pt 1pt](11.78,-1.08)(15.42,1.28)
\psline[linewidth=0.4pt](-0.12,-0.12)(4.78,0.32)
\psline[linewidth=0.4pt](8.43,0.65)(12.32,1)
\psline[linewidth=0.4pt](15.42,1.28)(21.14,1.8)
\psline[linewidth=0.4pt](-0.26,4.2)(5.9,2.18)
\psline[linewidth=0.4pt](8.43,1.35)(11.77,0.25)
\psline[linewidth=0.4pt](14.3,-0.58)(21.56,-2.97)
\psline[linewidth=0.4pt](14.3,-0.58)(11.77,0.25)
\psline[linewidth=0.4pt]{->}(10.22,2.94)(10.78,3.76)
\rput[tl](9.74,3.84){$\pi_\infty $}
\psdots[dotstyle=*](5.9,2.18)
\rput[bl](5.25,2.74){$X_1$}
\psdots[dotstyle=*](4.78,0.32)
\rput[bl](4.33,0.78){$X_2$}
\psdots[dotstyle=*](8.44,-1.32)
\rput[bl](7.58,-1.46){$X_3$}
\psdots[dotstyle=*](8.42,2.68)
\rput[bl](8.79,2.63){$X_4$}
\psdots[dotstyle=*](10.1,0.8)
\rput[bl](9.83,0.1){$S$}
\psdots[dotstyle=*,linecolor=darkgray](11.76,2.92)
\rput[bl](11.06,2.94){\darkgray{$Y_3$}}
\psdots[dotstyle=*,linecolor=darkgray](11.78,-1.08)
\rput[bl](11.08,-1.31){\darkgray{$Y_4$}}
\psdots[dotstyle=*,linecolor=darkgray](14.3,-0.58)
\rput[bl](14.56,-1.27){\darkgray{$Y_1$}}
\psdots[dotstyle=*,linecolor=darkgray](15.42,1.28)
\rput[bl](15.41,1.62){\darkgray{$Y_2$}}
\end{pspicture*}
\end{center}
\caption{Extended Desarguesian configuration for central symmetry}
\label{fig:Desarguesian configuration central symmetry}
\end{figure}
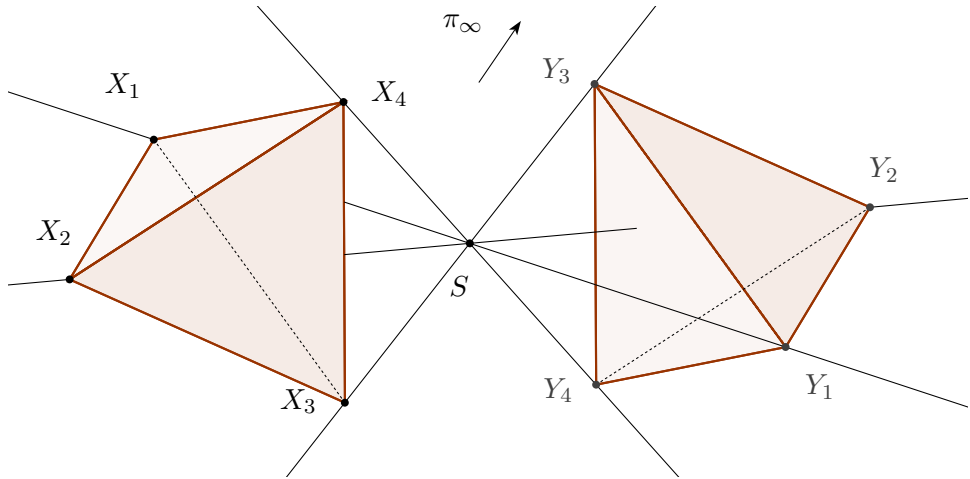

\begin{definition}[Space elation]\label{space elation}
A space elation is a nonsingular stereohomology, of which both the stereohomology center $S$ and the stereohomology hyperplane $\pi$ are ordinary (finite),  and $S$ $\in$ $\pi$.

See figure~\ref{fig:Desarguesian configuration space elation} for an intuitive concept of such a geometric transformation and definition No. 10 in table \ref{classification table} for the space elation formulation.
\end{definition}

Note that space elation follows the same rule as the convex and concave lens in elementary optics which is what inspires the extension to Desarguesian theorem and finally leads to the concept of stereohomology~\cite{investigation}.

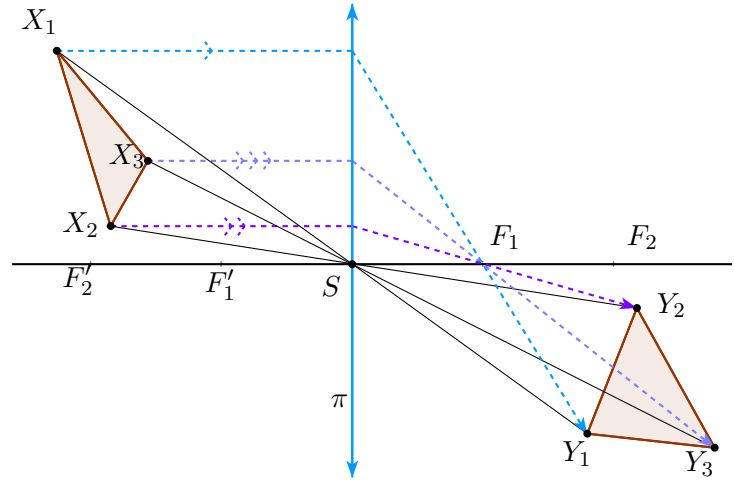
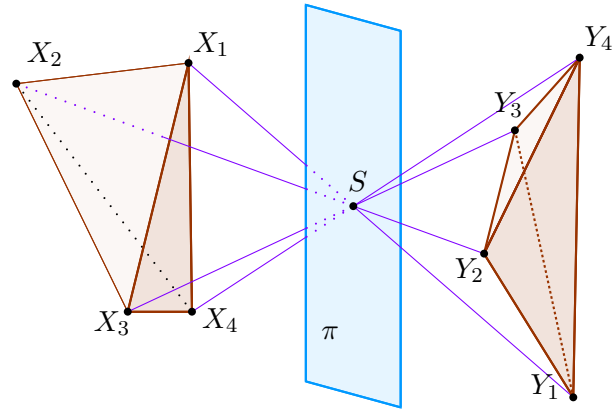
\begin{figure}[!hpt]
\begin{center}
\newrgbcolor{zzttqq}{0.6 0.2 0}
\newrgbcolor{yqqqqq}{0.5 0 0}
\newrgbcolor{qqzzff}{0 0.6 1}
\newrgbcolor{xdxdff}{0.49 0.49 1}
\newrgbcolor{xfqqff}{0.5 0 1}
\subfigure[Desarguesian configuration for 2D space elation]{\label{fig:a}
\psset{xunit=1.10cm,yunit=1.10cm,algebraic=true,dimen=middle,dotstyle=o,dotsize=3pt 0,linewidth=0.8pt,arrowsize=3pt 2,arrowinset=0.25}
\begin{pspicture*}(-4.11,-2.82)(4.59,3.2)
\pspolygon[linecolor=zzttqq,fillcolor=zzttqq,fillstyle=solid,opacity=0.1](-3.57,2.57)(-2.92,0.45)(-2.47,1.24)
\pspolygon[linecolor=zzttqq,fillcolor=zzttqq,fillstyle=solid,opacity=0.1](2.84,-2.06)(3.44,-0.54)(4.38,-2.23)
\psplot{-4.11}{4.59}{(-0.01-0*x)/1}
\psline[linecolor=zzttqq](-3.57,2.57)(-2.92,0.45)
\psline[linecolor=zzttqq](-2.92,0.45)(-2.47,1.24)
\psline[linecolor=zzttqq](-2.47,1.24)(-3.57,2.57)
\psline[linestyle=dashed,dash=2pt 2pt,linecolor=qqzzff](-3.57,2.57)(0,2.57)
\psline[linestyle=dashed,dash=2pt 2pt,linecolor=qqzzff](-1.7,2.57)(-1.78,2.47)
\psline[linestyle=dashed,dash=2pt 2pt,linecolor=qqzzff](-1.7,2.57)(-1.78,2.68)
\psline[linestyle=dashed,dash=2pt 2pt,linecolor=xdxdff](-2.47,1.24)(0,1.24)
\psline[linestyle=dashed,dash=2pt 2pt,linecolor=xdxdff](-1.15,1.24)(-1.24,1.14)
\psline[linestyle=dashed,dash=2pt 2pt,linecolor=xdxdff](-1.15,1.24)(-1.24,1.35)
\psline[linestyle=dashed,dash=2pt 2pt,linecolor=xdxdff](-1.32,1.24)(-1.4,1.14)
\psline[linestyle=dashed,dash=2pt 2pt,linecolor=xdxdff](-1.32,1.24)(-1.4,1.35)
\psline[linestyle=dashed,dash=2pt 2pt,linecolor=xdxdff](-0.99,1.24)(-1.07,1.14)
\psline[linestyle=dashed,dash=2pt 2pt,linecolor=xdxdff](-0.99,1.24)(-1.07,1.35)
\psline[linestyle=dashed,dash=2pt 2pt,linecolor=xfqqff](-2.92,0.45)(0,0.45)
\psline[linestyle=dashed,dash=2pt 2pt,linecolor=xfqqff](-1.3,0.45)(-1.38,0.34)
\psline[linestyle=dashed,dash=2pt 2pt,linecolor=xfqqff](-1.3,0.45)(-1.38,0.55)
\psline[linestyle=dashed,dash=2pt 2pt,linecolor=xfqqff](-1.46,0.45)(-1.54,0.34)
\psline[linestyle=dashed,dash=2pt 2pt,linecolor=xfqqff](-1.46,0.45)(-1.54,0.55)
\psline[linecolor=zzttqq](2.84,-2.06)(3.44,-0.54)
\psline[linecolor=zzttqq](3.44,-0.54)(4.38,-2.23)
\psline[linecolor=zzttqq](4.38,-2.23)(2.84,-2.06)
\psline[linewidth=0.4pt](-2.47,1.24)(4.38,-2.23)
\psline[linewidth=0.4pt](2.84,-2.06)(-3.57,2.57)
\psline[linewidth=0.4pt](-2.92,0.45)(3.44,-0.54)
\psline[linestyle=dashed,dash=2pt 2pt,linecolor=qqzzff]{->}(0,2.57)(2.84,-2.07)
\psline[linestyle=dashed,dash=2pt 2pt,linecolor=xdxdff]{->}(0,1.24)(4.38,-2.23)
\psline[linestyle=dashed,dash=2pt 2pt,linecolor=xfqqff]{->}(0,0.45)(3.44,-0.54)
\psline[linecolor=qqzzff]{->}(0,-2.59)(0,3.14)
\psline[linecolor=qqzzff]{->}(0,3.14)(0,-2.59)
\rput[tl](-0.26,-1.59){$\pi$}
\psdots[dotstyle=*](0,-0.01)
\rput[bl](-0.37,-0.4){$S$}
\psdots[dotstyle=*](-3.57,2.57)
\rput[bl](-4,2.78){$X_1$}
\psdots[dotstyle=*](-2.92,0.45)
\rput[bl](-3.5,0.32){$X_2$}
\psdots[dotstyle=*](-2.47,1.24)
\rput[bl](-2.94,1.15){$X_3$}
\psdots[dotsize=2pt 0,dotstyle=+](-1.58,-0.01)
\rput[bl](-1.78,-0.44){$F'_1$}
\psdots[dotsize=2pt 0,dotstyle=+](1.58,-0.01)
\rput[bl](1.66,0.17){$F_1$}
\psdots[dotsize=2pt 0,dotstyle=+](-3.16,-0.01)
\rput[bl](-3.5,-0.38){$F'_2$}
\psdots[dotsize=2pt 0,dotstyle=+](3.16,-0.01)
\rput[bl](3.32,0.17){$F_2$}
\psdots[dotstyle=*](3.44,-0.54)
\rput[bl](3.68,-0.65){$Y_2$}
\psdots[dotstyle=*](4.38,-2.23)
\rput[bl](4.02,-2.57){$Y_3$}
\psdots[dotstyle=*](2.84,-2.06)
\rput[bl](2.54,-2.45){$Y_1$}
\end{pspicture*}
}
\subfigure[Extended Desarguesian configuration for 3D space elation]{\label{fig:b}
\psset{xunit=0.101cm,yunit=0.101cm,algebraic=true,dimen=middle,dotstyle=o,dotsize=3pt 0,linewidth=0.8pt,arrowsize=3pt 2,arrowinset=0.25}
\begin{pspicture*}(-39.42,-28.12)(43.5,29.73)
\pspolygon[linewidth=0.4pt,linecolor=zzttqq,fillcolor=zzttqq,fillstyle=solid,opacity=0.04](-22.05,-13.92)(-13.59,-13.92)(-14.04,18.85)(-36.7,16.13)
\pspolygon[linewidth=0.4pt,linecolor=zzttqq,fillcolor=zzttqq,fillstyle=solid,opacity=0.04](36.6,-25.18)(24.85,-6.26)(28.93,9.99)(37.45,19.55)
\pspolygon[linecolor=zzttqq,fillcolor=zzttqq,fillstyle=solid,opacity=0.1](-22.05,-13.92)(-13.59,-13.92)(-14.04,18.85)
\pspolygon[linecolor=qqzzff,fillcolor=qqzzff,fillstyle=solid,opacity=0.1](13.88,23.09)(1.4,26.51)(1.4,-23.09)(13.88,-26.51)
\pspolygon[linecolor=zzttqq,fillcolor=zzttqq,fillstyle=solid,opacity=0.1](36.6,-25.18)(24.85,-6.26)(37.45,19.55)
\psline[linewidth=0.4pt,linecolor=zzttqq](-22.05,-13.92)(-13.59,-13.92)
\psline[linewidth=0.4pt,linecolor=zzttqq](-13.59,-13.92)(-14.04,18.85)
\psline[linewidth=0.4pt,linecolor=zzttqq](-14.04,18.85)(-36.7,16.13)
\psline[linewidth=0.4pt,linecolor=zzttqq](-36.7,16.13)(-22.05,-13.92)
\psline[linewidth=0.4pt,linecolor=zzttqq](36.6,-25.18)(24.85,-6.26)
\psline[linecolor=zzttqq](24.85,-6.26)(28.93,9.99)
\psline[linecolor=zzttqq](28.93,9.99)(37.45,19.55)
\psline[linewidth=0.4pt,linecolor=zzttqq](37.45,19.55)(36.6,-25.18)
\psline[linewidth=0.4pt](-22.05,-13.92)(-14.04,18.85)
\psline[linestyle=dotted](-36.7,16.13)(-13.59,-13.92)
\psline[linecolor=zzttqq](-22.05,-13.92)(-13.59,-13.92)
\psline[linecolor=zzttqq](-13.59,-13.92)(-14.04,18.85)
\psline[linecolor=zzttqq](-14.04,18.85)(-22.05,-13.92)
\psline[linestyle=dashed,dash=1pt 1pt,linecolor=zzttqq](36.6,-25.18)(28.93,9.99)
\psline[linecolor=yqqqqq](24.85,-6.26)(37.45,19.55)
\psline[linecolor=qqzzff](13.88,23.09)(1.4,26.51)
\psline[linecolor=qqzzff](1.4,26.51)(1.4,-23.09)
\psline[linecolor=qqzzff](1.4,-23.09)(13.88,-26.51)
\psline[linecolor=qqzzff](13.88,-26.51)(13.88,23.09)
\psline[linewidth=0.4pt,linecolor=xfqqff](24.85,-6.26)(7.64,0)
\psline[linestyle=dotted,linecolor=xfqqff](-17.37,9.1)(-36.7,16.13)
\psline[linewidth=0.4pt,linecolor=xfqqff](1.4,5.43)(-14.04,18.85)
\psline[linestyle=dotted,linecolor=xfqqff](1.4,5.43)(7.64,0)
\psline[linewidth=0.4pt,linecolor=xfqqff](7.64,0)(36.6,-25.18)
\psline[linewidth=0.4pt,linecolor=xfqqff](-17.37,9.1)(1.4,2.27)
\psline[linestyle=dotted,linecolor=xfqqff](1.4,2.27)(7.64,0)
\psline[linewidth=0.4pt,linecolor=xfqqff](7.64,0)(28.93,9.99)
\psline[linewidth=0.4pt,linecolor=xfqqff](1.4,-2.93)(-22.05,-13.92)
\psline[linestyle=dotted,linecolor=xfqqff](1.4,-2.93)(7.64,0)
\psline[linewidth=0.4pt,linecolor=xfqqff](7.64,0)(37.45,19.55)
\psline[linestyle=dotted,linecolor=xfqqff](7.64,0)(1.4,-4.1)
\psline[linewidth=0.4pt,linecolor=xfqqff](1.4,-4.1)(-13.59,-13.92)
\rput[tl](3.48,-15.89){$\pi $}
\psline[linecolor=zzttqq](36.6,-25.18)(24.85,-6.26)
\psline[linecolor=zzttqq](24.85,-6.26)(37.45,19.55)
\psline[linecolor=zzttqq](37.45,19.55)(36.6,-25.18)
\psdots[dotstyle=*](7.64,0)
\rput[bl](6.95,1.94){$S$}
\psdots[dotstyle=*](-22.05,-13.92)
\rput[bl](-26.5,-17){$X_3$}
\psdots[dotstyle=*](-13.59,-13.92)
\rput[bl](-12.23,-16.64){$X_4$}
\psdots[dotstyle=*](-14.04,18.85)
\rput[bl](-13.44,19.76){$X_1$}
\psdots[dotstyle=*](-36.7,16.13)
\rput[bl](-35.34,18.4){$X_2$}
\psdots[dotstyle=*](36.6,-25.18)
\rput[bl](30.97,-25.4){$Y_1$}
\psdots[dotstyle=*](24.85,-6.26)
\rput[bl](21,-10){$Y_2$}
\psdots[dotstyle=*](28.93,9.99)
\rput[bl](26.13,11.75){$Y_3$}
\psdots[dotstyle=*](37.45,19.55)
\rput[bl](38.06,20.51){$Y_4$}
\end{pspicture*}}
\end{center}
\caption{2D and 3D Extended Desarguesian configurations for space elation}
\label{fig:Desarguesian configuration space elation}
\end{figure}

\begin{definition}[Shearing]\label{shearing}
A shearing is a nonsingular stereohomology, of which the stereohomology center $S$ is infinite while the stereohomology hyperplane $\pi$ is ordinary (finite), and $S$ $\in$ $\pi$.

See figure~\ref{fig:Desarguesian configuration 2D shearing} for a 2D case visualization and definition No. 11 in table \ref{classification table} for the shearing formulation.
\end{definition}

\begin{figure}[!hpt]
\begin{center}
\newrgbcolor{zzttqq}{0.6 0.2 0}
\newrgbcolor{yqqqyq}{0.5 0 0.5}
\newrgbcolor{xfqqff}{0.5 0 1}
\psset{xunit=1.0cm,yunit=1.0cm,algebraic=true,dimen=middle,dotstyle=o,dotsize=3pt 0,linewidth=0.8pt,arrowsize=3pt 2,arrowinset=0.25}
\begin{pspicture*}(-0.86,-0.92)(11.88,3.68)
\pspolygon[linecolor=zzttqq,fillcolor=zzttqq,fillstyle=solid,opacity=0.1](2.52,1.52)(1,3.04)(4.04,3.04)
\pspolygon[linecolor=zzttqq,fillcolor=zzttqq,fillstyle=solid,opacity=0.1](5.56,1.52)(10.12,3.04)(7.08,3.04)
\psline(-0.5,0)(11.5,0)
\psline[linecolor=zzttqq](2.52,1.52)(1,3.04)
\psline[linecolor=zzttqq](1,3.04)(4.04,3.04)
\psline[linecolor=zzttqq](4.04,3.04)(2.52,1.52)
\psline[linecolor=zzttqq](5.56,1.52)(10.12,3.04)
\psline[linecolor=zzttqq](10.12,3.04)(7.08,3.04)
\psline[linecolor=zzttqq](7.08,3.04)(5.56,1.52)
\psline[linewidth=0.4pt,linestyle=dashed,dash=1pt 1pt,linecolor=yqqqyq](2.52,1.52)(4.04,0)
\psline[linewidth=0.4pt,linestyle=dashed,dash=1pt 1pt,linecolor=yqqqyq](1,0)(2.52,1.52)
\psline[linewidth=0.4pt,linestyle=dashed,dash=1pt 1pt,linecolor=yqqqyq](4.04,0)(5.56,1.52)
\psline[linewidth=0.4pt,linestyle=dashed,dash=1pt 1pt,linecolor=yqqqyq](1,0)(5.56,1.52)
\psline[linestyle=dashed,dash=1pt 1pt,linecolor=blue](1,0)(1,3.04)
\psline[linestyle=dashed,dash=1pt 1pt,linecolor=blue](4.04,3.04)(4.04,0)
\psline[linestyle=dashed,dash=1pt 1pt,linecolor=xfqqff](1,0)(7.08,3.04)
\psline[linestyle=dashed,dash=1pt 1pt,linecolor=xfqqff](4.04,0)(10.12,3.04)
\psline[linewidth=0.4pt](-0.5,1.52)(11.5,1.52)
\psline[linewidth=0.4pt](-0.5,3.04)(1,3.04)
\psline[linewidth=0.4pt](4.04,3.04)(7.08,3.04)
\psline[linewidth=0.4pt](10.12,3.04)(11.5,3.04)
\rput[tl](8.72,-0.08){$\pi$}
\psline{->}(0.46,1.82)(-0.26,1.82)
\psline{->}(10.5,1.82)(11.24,1.82)
\rput[tl](9.72,2.66){$S_{12}=S_{\infty }$}
\rput[tl](9.7,0.82){$S_{12}=S_{\infty }$}
\rput[tl](-0.6,2.64){$S_{12}=S_{\infty }$}
\rput[tl](-0.56,0.88){$S_{12}=S_{\infty }$}
\psdots[dotstyle=*](1,0)
\rput[bl](0.68,-0.6){$S_{23}$}
\psdots[dotstyle=*](4.04,0)
\rput[bl](4.18,-0.62){$S_{13}$}
\psdots[dotstyle=*,linecolor=darkgray](4.04,3.04)
\rput[bl](4.12,3.16){\darkgray{$X_2$}}
\psdots[dotstyle=*,linecolor=darkgray](1,3.04)
\psdots[dotstyle=*](1,3.04)
\rput[bl](1.08,3.16){$X_1$}
\psdots[dotstyle=*](7.08,3.04)
\rput[bl](7.16,3.16){$Y_1$}
\psdots[dotstyle=*](10.12,3.04)
\rput[bl](10.2,3.16){$Y_2$}
\psdots[dotstyle=*](2.52,1.52)
\rput[bl](2.36,1.86){$X_3$}
\psdots[dotstyle=*](5.56,1.52)
\rput[bl](5.38,1.88){$Y_3$}
\end{pspicture*}
\end{center}
\caption{Desarguesian configuration for 2D shearing}
\label{fig:Desarguesian configuration 2D shearing}
\end{figure}
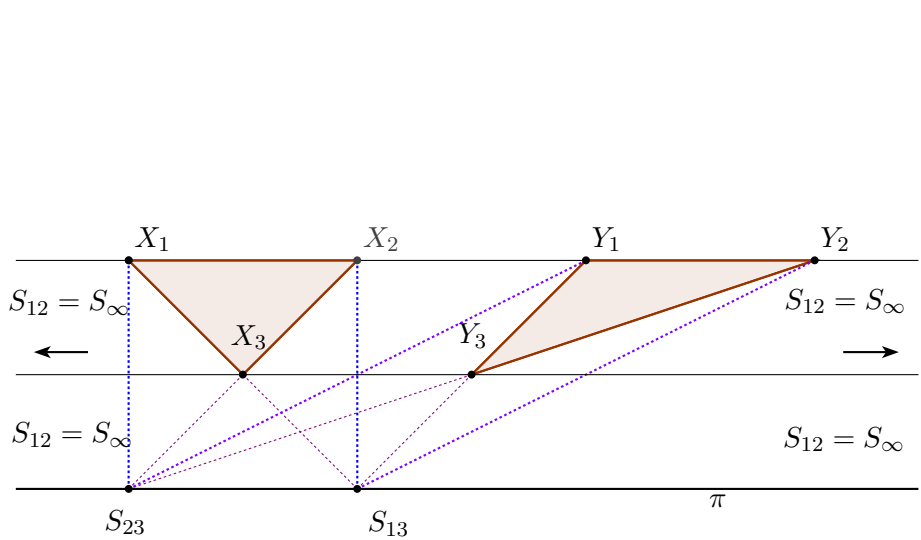

\begin{definition}[Translation]\label{translation}
A translation is a nonsingular stereohomology, of which both the stereohomology center $S$ and the stereohomology hyperplane $\pi$ are infinite elements, and  $S$ $\in$ $\pi$.

See figure~\ref{fig:Desarguesian configuration translation} for its extended Desarguesian configuration illustration and definition No. 12 in table \ref{classification table} for the translation formulation.
\end{definition}

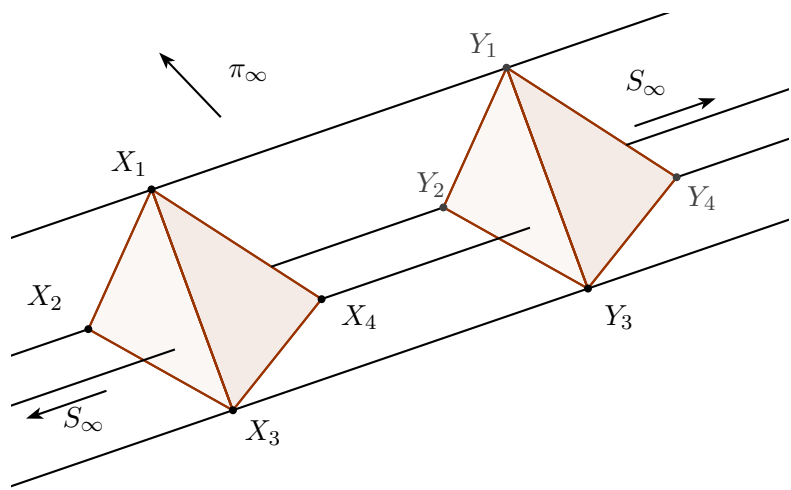
\begin{figure}[!hpt]
\begin{center}
\newrgbcolor{zzttqq}{0.6 0.2 0}
\psset{xunit=1.0cm,yunit=1.0cm,algebraic=true,dimen=middle,dotstyle=o,dotsize=3pt 0,linewidth=0.8pt,arrowsize=3pt 2,arrowinset=0.25}
\begin{pspicture*}(0.32,-1.66)(10.64,4.76)
\pspolygon[linecolor=zzttqq,fillcolor=zzttqq,fillstyle=solid,opacity=0.04](1.34,0.56)(2.18,2.42)(3.26,-0.52)
\pspolygon[linecolor=zzttqq,fillcolor=zzttqq,fillstyle=solid,opacity=0.1](3.26,-0.52)(4.44,0.96)(2.18,2.42)
\pspolygon[linecolor=zzttqq,fillcolor=zzttqq,fillstyle=solid,opacity=0.04](6.9,4.04)(6.06,2.18)(7.98,1.1)
\pspolygon[linecolor=zzttqq,fillcolor=zzttqq,fillstyle=solid,opacity=0.1](7.98,1.1)(9.16,2.58)(6.9,4.04)
\psline[linecolor=zzttqq](1.34,0.56)(2.18,2.42)
\psline[linecolor=zzttqq](2.18,2.42)(3.26,-0.52)
\psline[linecolor=zzttqq](3.26,-0.52)(1.34,0.56)
\psline[linecolor=zzttqq](3.26,-0.52)(4.44,0.96)
\psline[linecolor=zzttqq](4.44,0.96)(2.18,2.42)
\psline[linecolor=zzttqq](2.18,2.42)(3.26,-0.52)
\psline[linecolor=zzttqq](6.9,4.04)(6.06,2.18)
\psline[linecolor=zzttqq](6.06,2.18)(7.98,1.1)
\psline[linecolor=zzttqq](7.98,1.1)(6.9,4.04)
\psline[linecolor=zzttqq](7.98,1.1)(9.16,2.58)
\psline[linecolor=zzttqq](9.16,2.58)(6.9,4.04)
\psline[linecolor=zzttqq](6.9,4.04)(7.98,1.1)
\psline(-4.04,0.28)(13.15,6.18)
\psline(-4.11,-1.31)(1.34,0.56)
\psline(3.77,1.39)(6.06,2.18)
\psline(8.49,3.01)(16.62,5.8)
\psline(-3.89,-1.9)(2.49,0.29)
\psline(4.44,0.96)(7.21,1.91)
\psline(9.16,2.58)(17.67,5.5)
\psline(-3.74,-2.92)(18.07,4.56)
\psline{->}(1.58,-0.26)(0.51,-0.63)
\psline{->}(8.6,3.26)(9.68,3.63)
\rput[tl](8.48,4){$S_\infty $}
\rput[tl](1,-0.48){$S_\infty $}
\psline{->}(3.1,3.38)(2.28,4.24)
\rput[tl](3.2,4.08){$\pi _\infty $}
\psdots[dotstyle=*](2.18,2.42)
\rput[bl](1.64,2.6){$X_1$}
\psdots[dotstyle=*](1.34,0.56)
\rput[bl](0.52,0.82){$X_2$}
\psdots[dotstyle=*](3.26,-0.52)
\rput[bl](3.42,-0.98){$X_3$}
\psdots[dotstyle=*](4.44,0.96)
\rput[bl](4.7,0.56){$X_4$}
\psdots[dotstyle=*](7.98,1.1)
\rput[bl](8.2,0.56){$Y_3$}
\psdots[dotstyle=*,linecolor=darkgray](9.16,2.58)
\rput[bl](9.3,2.14){\darkgray{$Y_4$}}
\psdots[dotstyle=*,linecolor=darkgray](6.9,4.04)
\rput[bl](6.4,4.18){\darkgray{$Y_1$}}
\psdots[dotstyle=*,linecolor=darkgray](6.06,2.18)
\rput[bl](5.7,2.32){\darkgray{$Y_2$}}
\end{pspicture*}
\end{center}
\caption{Extended Desarguesian configuration for translation}
\label{fig:Desarguesian configuration translation}
\end{figure}

\subsection{Involutory Stereohomology Representation of Elementary Perspectives}

\begin{theorem}[Two involution theorem 1]\label{double involution1}
If $\mathscr{T}_1$ and $\mathscr{T}_2$ are two involutory stereohomology with stereohomology centers of $S_1$, $S_2$ , and share the same stereohomology hyperplane $\pi$, then $\mathscr{T}_3$ $=$ $\mathscr{T}_1$ $\cdot$ $\mathscr{T}_2$ and $\mathscr{T}_4$ $=$ $\mathscr{T}_2$ $\cdot$ $\mathscr{T}_1$ are a pair of elementary perspectives inverse to each other.

Additioinally, $\mathscr{T}_3$ and $\mathscr{T}_4$ share a common stereohomology hyperplane $\pi$ and a common stereohomology center which is the intersection of the line $S_1S_2$ and $\pi$.
\end{theorem}

Similarly, we also have:
\begin{theorem}[Two involution theorem 2]\label{double involution2}
If $\mathscr{T}_1$ and $\mathscr{T}_2$ are two involutory stereohomology with stereohomology hyperplanes $\pi_1$, $\pi_2$ , and share a common stereohomology center $S$, then  $\mathscr{T}_3$ $=$ $\mathscr{T}_1$ $\cdot$ $\mathscr{T}_2$ and $\mathscr{T}_4$ $=$ $\mathscr{T}_2$ $\cdot$ $\mathscr{T}_1$ are a pair of elementary perspectives inverse to each other.

Additioinally, $\mathscr{T}_3$ and $\mathscr{T}_4$ share a common stereohomology center $S$ and a common stereohomology hyperplane which is the joining of the line $\pi_1\cap\pi_2$ and $S$.
\end{theorem}

Based on theorems \ref{double involution1} and \ref{double involution2}, we present definitions for elementary perspectives which are different from but equivalent to their original definitions. The new definitions may provide convenience in obtaining desired elementary perspective matrices.

\begin{definition}[Translation]\label{translation1}
A translation is a compound transformation of two central symmetry of which the stereohomology centers $S_1$ and $S_2$ are different ordinary points, and which share a common infinite sterehomology hyperplane.

The displacement of the translation is twice the distance between $S_1$ and $S_2$.
\end{definition}

\begin{definition}[Translation]\label{translation2}
A translation is a compound transformation of two orthographic reflections of which the ordinary stereohomology hyperplanes $\pi_1$ and $\pi_2$ are different but share a common normal direction.

The displacement of the translation is twice the distance between $\pi_1$ and $\pi_2$.
\end{definition}

\begin{definition}[Shearing]\label{shearing1}
A shearing is a compound transformation of two involutory stereohomology, of which the two stereohomology centers $S_1$ and $S_2$ are different ordinary points, and which share a common ordinary sterehomology hyperplane $\pi$, if the line joining $S_1$ and $S_2$ intersects $\pi$ at infinite.

The shearing displacement is hence twice the distance between $S_1$ and $S_2$.
\end{definition}

\begin{definition}[Shearing]\label{shearing2}
A shearing is a compound transformation of two reflections, at least one of which is oblique, the two mirror hyperplanes $\pi_1$ and $\pi_2$ of which are different ordinary hyperplanes intersecting at an ordinary line, and which share a common infinite stereohomology center $S$.
\end{definition}

Except for projections, shearing and scaling\cite[p.11]{Faugeras} are also fundamental concepts in computer vision which were not rigorously defined independent of a particular choice of coordinate system before.

\subsection{Definitions and Representation of General Rotations}
The definitions and representation for {\em elementary geometric transformations} can be further extended to those geometric transformations which are not elementary since square matrices can be decomposed into elementary factors. Here we only present an example to general 3D rotations with axes not necessarily passing through coordinate system origin.

A rotation can be defined as the compound operation of two reflections according to~\cite[pp.419~\texttildelow~422]{Perspectives}, but there is no simple rotation representation derived based on such definitions yet. In this section we shall both use the compound transformation of two orthographic {\em reflections} defined as involutory stereohomology in table~\ref{classification table} and use the eigen-system of the rotation which is inherent algebraic features per theorem~\ref{zeroth}, to represent a general rotation.

The definition of an orthographic reflection in~\cite{meaning} takes advantage of the existence and uniqueness of an involutory projective transformation which transforms $X_i$ and $S$ in the extended Desargues configuration $X_1X_2X_3X_4-S-Y_1Y_2Y_3Y_4$ (as in figure~\ref{fig:Desarguesian configuration reflection}) into $Y_i$ and $S$ in sequence respectively. The homogeneous square matrix formulation of such a reflection was proved to be in the form as indicated in table~\ref{classification table}~\cite{meaning}.

Note that in order to make the definitions in algebraic projective geometry {\em compatible with} the Euclidean geometry intuitions in one's mind, we have to make choices to distinguish ordinary and infinite geometric elements which are algebraically indistinguishable in projective space. Then we redefine homogeneous rotations in ${\mathbb P}^n$ ( only when $n$ = 2,3) in this paper as:

\begin{definition}[Rotation]\label{rotation1}
A rotation in ${\mathbb P}^n$ is a compound transformation of two orthographic reflections of which the stereohomology centers $S_1$ and $S_2$ are different infinite points, and sterehomology hyperplanes $\pi_1$ and $\pi_2$ are ordinary elements(see table~\ref{classification table} for definitions of elementary geometric transformations).

The rotation angle $\theta$ of the rotation is twice that of dihedral angle $\omega$ between $\pi_1$ and $\pi_2$ which can be represented by Laguerre's formula by involving cross ratio~\cite[pp.342,409]{Perspectives}.
\end{definition}

The above definition~\ref{rotation1} is directly borrowed from the classic definitions in projective geometry, and is theoretically dependent on the possibility of defining {\em normal reflection} as an involutory stereohomology in table~\ref{classification table} where modified Householder's elementary matrices~\cite[1\texttildelow 3]{Householder} are presented and defined into {\em stereohomology} based on an extension to Desargues theorem~\cite[75\texttildelow 76]{Veblen}. Otherwise, we have not find any other opportunity of define rotation via such an approach in algebraic projective geometry. It is only based on definition~\ref{rotation1} that we can obtain a pure algebraic definition~\ref{rotation2} of 2D and 3D rotations in projective space without using any non-projective-geometry concept, i.e., it is logically inappropriate to immediately adopt a Givens matrix as a {\em standard} rotation. %

\begin{definition}[Rotation]\label{rotation2}
A rotation in ${\mathbb P}^n$($n$ = 2,3) with  rotation angle $\theta$  and rotation axis $l$ (the latter of which should be able to be represented as the intersection of two hyperplanes in ${\mathbb P}^n$) is a projective transformation of which:

(1) the ratios of all eigenvalues are cos$\theta \pm i\cdot$ sin$\theta $ and 1;

(2) points on rotation axis $l$ are the associated eigenvectors with the ratio 1 real eigenvalue, and

(3) the associated eigenvectors with eigenvalues of ratios cos$\theta \pm i\cdot$ sin$\theta $ are the intersetion points of the imaginary conics\cite[p.204]{Vaisman} and the infinite hyperplane in ${\mathbb P}^2$, or are the intersection points of the imaginary quadrics\cite[p.204]{Vaisman}, the infinite hyperplane, and any ordinary hyperplane of which the normal direction is the direction of the rotation axis when it is in ${\mathbb P}^3$.
\end{definition}

We shall obtain homogeneous rotations based on definitions \ref{rotation1} and \ref{rotation2} via two approaches different from those in~\cite[pp.33~\texttildelow~34,43~\texttildelow~48]{CAD}, \cite[pp.43~\texttildelow~52]{Salomon}, \cite[p.36]{VinceFormulae}, \cite[pp.89~\texttildelow~90,115~\texttildelow~118,177~\texttildelow~180]{VinceRotation}:
\begin{description}
\item[(I)] find two hyperplanes of which their intersection line being rotation axis and the dihedral angle $\omega$ being half the angle $\theta$, then the products of reflections about the two hyperplanes will be desired rotation and its inverse (per definition \ref{rotation1}), further characteristic geometric features of positive direction of rotation axes and the right- or left-handed rule, can finalize the desired rotation;
\item[(II)] find all the eigenvalues and their associate eigenvectors, then the rotation and its inverse can be obtained by reconstructing from its eigen-decomposition factors(per definition \ref{rotation2}), further characteristic geometric information on the rotation similar to above uniquely determines the rotation.
\end{description}

The right-handed 3D homogeneous rotation with rotation angle $\theta$ and rotation axis through $(x_0,y_0,z_0,1)^T$ with axis direction $(a,b,c,0)^T$ can therefore be obtained by the approaches above as in~\eqref{3D homogeneous roation}, which for application convenience has been rewritten into a user friendly form similar to the classic Rodrigues' formula~\cite[p.165]{Goldman} 
({\color{black}Note}: without loss of generality, we assume $\color{red}a^2+b^2+c^2=1$):
{
\begin{gather}\label{3D homogeneous roation}
\boldsymbol{R}^{3D}\left(x_0,y_0,z_0,a,b,c,\theta\right)=\mathscr{C}_1+\left(\sin\theta\cdot\mathscr{A}_2- \left(1-\cos\theta\right)\cdot\mathscr{O}_3\right)\cdot \mathscr{T}_4
\end{gather}
}
where:
\begin{flalign*}
\quad\mathscr{C}_1=
\begin{array}{c}
\underbrace{\left[
\begin{array}{*{20}{c}}
1&0&0&0\\
0&1&0&0\\
0&0&1&0\\
0&0&0&{2 - \cos \theta }
\end{array} \right]} \\ \text{Central dilation}{_{\tiny\tfrac{1}{ (2-\cos\theta)}}}
\end{array},  & \qquad\quad\mathscr{A}_2  =  {\begin{array}{c}
\underbrace{
\begin{array}{c}
 {{\left[ {\begin{array}{*{20}{c}}
{\color{blue}0}&{\color{blue}-c}&{\color{blue}b}&0\\
{\color{blue}c}&{\color{blue}0}&{\color{blue}-a}&0\\
{\color{blue}-b}&{\color{blue}a}&{\color{blue}0}&0\\
0&0&0&0
\end{array}} \right]}}
\end{array}
} \\ \text{Antisymmetric matrix}
\end{array}}
\end{flalign*}

\begin{flalign*}
\quad\mathscr{O}_3= \begin{array}{c}
\underbrace{  {I - {\left[ {\begin{array}{*{20}{c}}
a\\
b\\
c\\
0
\end{array}} \right] \cdot \left[ {\begin{array}{*{20}{c}}
a&b&c&0
\end{array}} \right]}} }\\
\text{Orthographic parallel projection} \end{array}, \quad & \mathscr{T}_4=\begin{array}{c}
\underbrace{\left[ {\begin{array}{*{20}{c}}
1&0&0&{ - {x_0}}\\
0&1&0&{ - {y_0}}\\
0&0&1&{ - {z_0}}\\
0&0&0&1
\end{array}}  \right]}\\ \text{Translation}\end{array}
\end{flalign*}

\section{Conclusions}
We have algebraically redefined as stereohomology a series of geometric transformations, i.e., central projection, parallel projection, direction, space homology, elementary scaling, central dilation, reflection, central symmetry, space elation, shearing and translation, and represented them into modified Householder elementary matrices. Such results can be further extended to general homogeneous geometric transformations which are not elementary.

\section*{Acknowledgements}

\clearpage

\section*{Appendices}
\begin{proposition}[Existence and uniqueness of stereohomology]
There exists a unique generalized collineation(or generalized projective transformation, see Definition~\ref{Desargues} in homogeneous transformation matrix which transforms $X_1X_2\cdots X_n$ and $S$ in an extended Desarguesian configuration into $Y_1Y_2\cdots Y_n$ and $S$(or null), respectively.
\end{proposition}

\begin{proof}

Since the applications are mainly those geometric transformations in 3D and 2D projective spaces, we prove the cases in 3D projective space only.

When all $Y_1Y_2\cdots Y_n$ are not coplanar, the proof is similar to that of {\em the first fundamental theorem in projective geometry}(see page~\hyperlink{page.5}{5} of {\em The Geometry of Multiple Images, 2001}, reference~\cite{Faugeras} of the manuscript), that is, by taking advantage of $O_1O_2\cdots O_nO_{n+1}$ as transition. That is, we first construct the generalized collineation $T_1$ which transforms $O_i$'s into $X_i$'s and $S$ respectively, then $T_2$ which transforms $O_i$'s into $Y_i$'s and $S$ respectively. The combination $T = T_1^{-1}\cdot T_2$ should be the desired generalized collineation (projective transformation) which transforms $X_i$'s and $S$ into $Y_i$'s and $S$(or null) respectively, namely the {\em stereohomology}. The uniqueness can also be proved from the construction.

Here we use the extended Desuargues configuration $ABCD\text{-}S\text{-}A'B'C'D'$ and the notations as in page~\hyperlink{page.5}{5} of the manuscript. Denote $O_1:$ $(1,0,0,0)^T$, $O_2:$ $(0,1,0,0)^T$, $O_3:$ $(0,0,1,0)^T$, $O_4:$ $(0,0,0,1)^T$, $O_5:$ $(1,1,1,1)^T$. Denote the generalized collineation which transforms $O_i$s into $ABCDS$ respectively as $t_{ij}$ $(i,j=1,2,3,4)$. Then

\begin{equation}\label{eqn:original transformation}
\left(\begin{array}{ccccc}
 {a_1} {\rho_1} & {b_1} {\rho_2} & {c_1}
   {\rho_3} & {d_1} {\rho_4} & {s_1} {\rho_5} \\
 {a_2} {\rho_1} & {b_2} {\rho_2} & {c_2}
   {\rho_3} & {d_2} {\rho_4} & {s_2} {\rho_5} \\
 {a_3} {\rho_1} & {b_3} {\rho_2} & {c_3}
   {\rho_3} & {d_3} {\rho_4} & {s_3} {\rho_5} \\
 {a_4} {\rho_1} & {b_4} {\rho_2} & {c_4}
   {\rho_3} & {d_4} {\rho_4} & {s_4} {\rho_5} \\
\end{array}
\right)=\left(t_{ij}\right)_{4\times 4}\cdot
\left(
\begin{array}{ccccc}
 1 & 0 & 0 & 0 & 1 \\
 0 & 1 & 0 & 0 & 1 \\
 0 & 0 & 1 & 0 & 1 \\
 0 & 0 & 0 & 1 & 1 \\
\end{array}
\right)
\end{equation}
without loss of generality, let $\rho_1 = 1$, then equation~\eqref{eqn:original transformation} becomes 20 linear equations with 20 variables: $\rho_i$ ($i=2,3,4,5$), and $t_{ij}$ ($i,j=1,2,3,4$) as in~\eqref{eqn:original transformation with rho=1}:

\begin{equation}\label{eqn:original transformation with rho=1}
\left(\begin{array}{ccccc}
 {a_1} & {b_1} {\rho_2} & {c_1}
   {\rho_3} & {d_1} {\rho_4} & {s_1} {\rho_5} \\
 {a_2} & {b_2} {\rho_2} & {c_2}
   {\rho_3} & {d_2} {\rho_4} & {s_2} {\rho_5} \\
 {a_3} & {b_3} {\rho_2} & {c_3}
   {\rho_3} & {d_3} {\rho_4} & {s_3} {\rho_5} \\
 {a_4} & {b_4} {\rho_2} & {c_4}
   {\rho_3} & {d_4} {\rho_4} & {s_4} {\rho_5} \\
\end{array}
\right)=\left(
\begin{array}{cccc}
 {t_{11}} & {t_{12}} & {t_{13}} &{t_{14}} \\
 {t_{21}} & {t_{22}} & {t_{23}} &{t_{24}} \\
 {t_{31}} & {t_{32}} & {t_{33}} &{t_{34}} \\
 {t_{41}} & {t_{42}} & {t_{43}} &{t_{44}} \\
\end{array}
\right)\cdot
\left(
\begin{array}{ccccc}
 1 & 0 & 0 & 0 & 1 \\
 0 & 1 & 0 & 0 & 1 \\
 0 & 0 & 1 & 0 & 1 \\
 0 & 0 & 0 & 1 & 1 \\
\end{array}
\right)
\end{equation}
which can be rewritten into equations~\eqref{eqn: transform expansion 01}\texttildelow\eqref{eqn: transform expansion 05} as below:
\begin{align}
t_{11}=\rho_1 a_1=a_1,\quad t_{21}=\rho_1 a_2=a_2,\quad t_{31}=\rho_1 a_3=a_3,\quad t_{41}=\rho_1 a_4=a_4 \label{eqn: transform expansion 01}\\
t_{12}-b_1\rho_2=0,\quad t_{22}-b_2\rho_2=0,\quad t_{32}-b_3\rho_2=0,\quad t_{42}-b_4\rho_2=0 \label{eqn: transform expansion 02}\\
t_{13}-c_1\rho_3=0,\quad t_{23}-c_2\rho_3=0,\quad t_{33}-c_3\rho_3=0,\quad t_{43}-c_4\rho_3=0\label{eqn: transform expansion 03}\\
t_{14}-d_1\rho_4=0,\quad t_{24}-d_2\rho_4=0,\quad t_{34}-d_3\rho_4=0,\quad t_{44}-d_4\rho_4=0\label{eqn: transform expansion 04}\\
\left(\sum\limits_{j=1}^4{t_{ij}}\right) -s_i\rho_5=0,\quad i=1,2,3,4 \label{eqn: transform expansion 05}
\end{align}
and rewritten into the $(n+1)(n+2)-$dimensional linear system:
\begin{equation}\label{eqn:20 dimensional linear system}
\left(
\begin{array}{cccccccccccccccccccc}
 1 & 0 & 0 & 0 & 0 & 0 & 0 & 0 & 0 & 0 & 0 & 0 & 0 & 0 & 0 & 0 & 0 & 0 & 0 & 0 \\
 0 & 1 & 0 & 0 & 0 & 0 & 0 & 0 & 0 & 0 & 0 & 0 & 0 & 0 & 0 & 0 & 0 & 0 & 0 & 0 \\
 0 & 0 & 1 & 0 & 0 & 0 & 0 & 0 & 0 & 0 & 0 & 0 & 0 & 0 & 0 & 0 & 0 & 0 & 0 & 0 \\
 0 & 0 & 0 & 1 & 0 & 0 & 0 & 0 & 0 & 0 & 0 & 0 & 0 & 0 & 0 & 0 & 0 & 0 & 0 & 0 \\
 0 & 0 & 0 & 0 & 1 & 0 & 0 & 0 & 0 & 0 & 0 & 0 & 0 & 0 & 0 & 0 & -{b_1} & 0 & 0 & 0 \\
 0 & 0 & 0 & 0 & 0 & 1 & 0 & 0 & 0 & 0 & 0 & 0 & 0 & 0 & 0 & 0 & -{b_2} & 0 & 0 & 0 \\
 0 & 0 & 0 & 0 & 0 & 0 & 1 & 0 & 0 & 0 & 0 & 0 & 0 & 0 & 0 & 0 & -{b_3} & 0 & 0 & 0 \\
 0 & 0 & 0 & 0 & 0 & 0 & 0 & 1 & 0 & 0 & 0 & 0 & 0 & 0 & 0 & 0 & -{b_4} & 0 & 0 & 0 \\
 0 & 0 & 0 & 0 & 0 & 0 & 0 & 0 & 1 & 0 & 0 & 0 & 0 & 0 & 0 & 0 & 0 & -{c_1} & 0 & 0 \\
 0 & 0 & 0 & 0 & 0 & 0 & 0 & 0 & 0 & 1 & 0 & 0 & 0 & 0 & 0 & 0 & 0 & -{c_2} & 0 & 0 \\
 0 & 0 & 0 & 0 & 0 & 0 & 0 & 0 & 0 & 0 & 1 & 0 & 0 & 0 & 0 & 0 & 0 & -{c_3} & 0 & 0 \\
 0 & 0 & 0 & 0 & 0 & 0 & 0 & 0 & 0 & 0 & 0 & 1 & 0 & 0 & 0 & 0 & 0 & -{c_4} & 0 & 0 \\
 0 & 0 & 0 & 0 & 0 & 0 & 0 & 0 & 0 & 0 & 0 & 0 & 1 & 0 & 0 & 0 & 0 & 0 & -{d_1} & 0 \\
 0 & 0 & 0 & 0 & 0 & 0 & 0 & 0 & 0 & 0 & 0 & 0 & 0 & 1 & 0 & 0 & 0 & 0 & -{d_2} & 0 \\
 0 & 0 & 0 & 0 & 0 & 0 & 0 & 0 & 0 & 0 & 0 & 0 & 0 & 0 & 1 & 0 & 0 & 0 & -{d_3} & 0 \\
 0 & 0 & 0 & 0 & 0 & 0 & 0 & 0 & 0 & 0 & 0 & 0 & 0 & 0 & 0 & 1 & 0 & 0 & -{d_4} & 0 \\
 1 & 0 & 0 & 0 & 1 & 0 & 0 & 0 & 1 & 0 & 0 & 0 & 1 & 0 & 0 & 0 & 0 & 0 & 0 & -{s_1} \\
 0 & 1 & 0 & 0 & 0 & 1 & 0 & 0 & 0 & 1 & 0 & 0 & 0 & 1 & 0 & 0 & 0 & 0 & 0 & -{s_2} \\
 0 & 0 & 1 & 0 & 0 & 0 & 1 & 0 & 0 & 0 & 1 & 0 & 0 & 0 & 1 & 0 & 0 & 0 & 0 & -{s_3} \\
 0 & 0 & 0 & 1 & 0 & 0 & 0 & 1 & 0 & 0 & 0 & 1 & 0 & 0 & 0 & 1 & 0 & 0 & 0 & -{s_4} \\
\end{array}
\right)
\left[
\begin{array}{c}
{t_{11}}\\
{t_{21}}\\
{t_{31}}\\
{t_{41}}\\
{t_{12}}\\
{t_{22}}\\
{t_{32}}\\
{t_{42}}\\
{t_{13}}\\
{t_{23}}\\
{t_{33}}\\
{t_{43}}\\
{t_{14}}\\
{t_{24}}\\
{t_{34}}\\
{t_{44}}\\
{\rho_2}\\
{\rho_3}\\
{\rho_4}\\
{\rho_5}\\
\end{array}
\right]=
\left[
\begin{array}{c}
 {a_1} \\
 {a_2} \\
 {a_3} \\
 {a_4} \\
 0 \\
 0 \\
 0 \\
 0 \\
 0 \\
 0 \\
 0 \\
 0 \\
 0 \\
 0 \\
 0 \\
 0 \\
 0 \\
 0 \\
 0 \\
 0 \\
\end{array}
\right]
\end{equation}

The $20\times 20$ coefficient matrix in~\eqref{eqn:20 dimensional linear system} is always nonsingular since its determinant is equal to minus that of the following matrix, which is nonsingular since $B,C,D\text{ and }S$ in an extended Desarguesian configuration are not coplanar:
\begin{equation}\label{eqn:matrix noncoplanar}
\left(
\begin{array}{cccc}
{b_1}&{c_1}&{d_1}&{s_1}\\
{b_2}&{c_2}&{d_2}&{s_2}\\
{b_3}&{c_3}&{d_3}&{s_3}\\
{b_4}&{c_4}&{d_4}&{s_4}\\
\end{array}
\right)\end{equation}

Since $a_i$ $(i=1,2,3,4)$ cannot be all zero, then there exist a unique nontrivial solution to the $20\times 20$ linear system. So does the desired generalized projective transformation which transforms $ABCDS$ into $A'B'C'D'S$.

The conclusion can be extended to $n$-dimensional projective spaces by examining the existence and uniqueness of the solution to a thus obtained $(n+1)*(n+2)$ dimensional linear system.
\end{proof}

Below we show how to obtain the generalized collineation in the form as in equation~\eqref{Desargues} in page~\hyperlink{page.5}{5} of the manuscript.

First, also from equations~\eqref{eqn: transform expansion 01}, \eqref{eqn: transform expansion 02}, \eqref{eqn: transform expansion 03}, \eqref{eqn: transform expansion 04} and \eqref{eqn: transform expansion 05}, we obtain the following relationships:
\begin{equation}\label{eqn:linear system for Gramer's rule}
\left\{\quad
\begin{array}{c}
 {\rho_1 a_1+\rho_2 b_1+\rho_3 c_1+\rho_4 d_1=\rho_5 s_1} \\
 {\rho_1 a_2+\rho_2 b_2+\rho_3 c_2+\rho_4 d_2=\rho_5 s_2} \\
 {\rho_1 a_3+\rho_2 b_3+\rho_3 c_3+\rho_4 d_3=\rho_5 s_3} \\
 {\rho_1 a_4+\rho_2 b_4+\rho_3 c_4+\rho_4 d_4=\rho_5 s_4} \\
\end{array}
\right.\end{equation}
Therefore: $\rho_1 : \rho_2 : \rho_3 : \rho_4 : \rho_5 = \Delta_1 : \Delta_2 : \Delta_3 : \Delta_4 : \Delta_5$ (per Cramer's rule), where:
\[{\small
\begin{gathered}
  \Delta _1 \! =\! \left| {\begin{array}{*{20}c}
   {s_1 }\! & {b_1 }\! & {c_1 }\! & {d_1 }  \\
   {s_2 }\! & {b_2 }\! & {c_2 }\! & {d_2 }  \\
   {s_3 }\! & {b_3 }\! & {c_3 }\! & {d_3 }  \\
   {s_4 }\! & {b_4 }\! & {c_4 }\! & {d_4 }  \\
 \end{array} } \right|,\:\Delta _2 \! = \!\left| {\begin{array}{*{20}c}
   {a_1 }\! & {s_1 }\! & {c_1 }\! & {d_1 }\!  \\
   {a_2 }\! & {s_2 }\! & {c_2 }\! & {d_2 }\!  \\
   {a_3 }\! & {s_3 }\! & {c_3 }\! & {d_3 }\!  \\
   {a_4 }\! & {s_4 }\! & {c_4 }\! & {d_4 }\!  \\
 \end{array} } \right|,
  \Delta _3 \! =\! \left| {\begin{array}{*{20}c}
   {a_1 }\! & {b_1 }\! & {s_1 }\! & {d_1 }\!  \\
   {a_2 }\! & {b_2 }\! & {s_2 }\! & {d_2 }\!  \\
   {a_3 }\! & {b_3 }\! & {s_3 }\! & {d_3 }\!  \\
   {a_4 }\! & {b_4 }\! & {s_4 }\! & {d_4 }\!
 \end{array} } \right|, \Delta _4 \! = \! \left| {\begin{array}{*{20}c}
   {a_1 }\! & {b_1 }\! & {c_1 }\! & {s_1 }  \\
   {a_2 }\! & {b_2 }\! & {c_2 }\! & {s_2 }  \\
   {a_3 }\! & {b_3 }\! & {c_3 }\! & {s_3 }  \\
   {a_4 }\! & {b_4 }\! & {c_4 }\! & {s_4 }
 \end{array} } \right| \hfill \\
\end{gathered}
}\]

Let $\rho_1 = k\cdot \Delta_1 ( 0\ne k\in\mathbb{R})$, then $\rho_i = k\cdot \Delta_i$ $(i=1,2,3,4,5)$, and $\mathscr{T}$ in equation~({\color{blue}2.3}) (page {\color{blue}5} of the manuscript) can be easily obtained.

\begin{proposition}[Existence of an $n$-dimensional Eigenspace]\label{exist-pi} The $n$+1 dimensional homogeneous matrix of generalized projective transformation $\mathscr{T}$ as in equation ({\color{blue}2.3}) of the manuscript, has an eigenvalue with geometric multiplicity of $n$, i.e., there exists an $n$-dimensional eigenspace (invariant subspace, invariant hyperplane) of the transformation $\mathscr{T}$.
\end{proposition}

\begin{proof} Let $\mathscr{T}$ represent both the {\em geometric transformation} and the {\em transformation matrix.} Only prove the case when $\mathscr{T}$ is nonsingular, i.e., both the homogeneous coordinates of $(x_i)$ and those of $(y_i)$ are linearly independent. The proof is based on a constructed minimal polynomial of $\mathscr{T}$.

Since he homogeneous coordinate of $S$ can be linearly expressed as:
\begin{eqnarray*}
\begin{array}{ccccc}
(s)\quad &=&\lambda_1 (x)_1&+&\lambda'_1 (y)_1\\
  &=&\lambda_2 (x)_2&+&\lambda'_2 (y)_2\\
  &\vdots &\vdots & \vdots & \vdots \nonumber\\
   &=&\lambda_{n+1} (x)_{n+1}&+& \lambda'_{n+1} (y)_{n+1}
\end{array}\\[3pt]
\end{eqnarray*}
and since the $n$+1 homogeneous coordinates of $(x_i)$ and $(y_i)$ ( $i$ = 1,$\cdots$, $n$+1 ) are linearly independent, $\exists$ $0 \neq$ $\mu_i$ $\in\mathbb{R}$ ($i$ = 1,$\cdots$, $n$+1 ), which make $(s)$ can be linearly expressed as:

\begin{equation*} \label{s-sum}
(s)= \sum\limits_{i = 1}^{n+1} {\mu_i\;(x_i) }\\
\end{equation*}

 Since $(x_i)$ and $(y_i)$ are the corresponding points through the geometric transformation $\mathscr{T}$ $(\forall\;i\,=\,1,\cdots,n+1)$, i.e., there exist: 0 $\neq$ $\rho_i$ $\in$ $\mathbb{R}$ ($i$\,= $1,\ldots$,$n+1$),  which satisfy:
\begin{equation}\label{x-rho}
\rho_i\,(y_i)\;=\; \mathscr{T} (x_i)\qquad(i\,= 1,\cdots,n+1)
\end{equation}

\begin{equation}\label{s-rho}
\text{and}\quad \rho_s\,(s)\;=\; \mathscr{T} (s) \qquad \qquad
\qquad \qquad \end{equation}

By combining the  results in equations  \eqref{x-rho} and \eqref{s-rho} here together, the linear expression of $(s)$ by $(y_i)$  can be obtained in two different forms:
\begin{equation*}
\left( {\sum\limits_{i = 1}^{n + 1} {\frac{{\mu _i }}{{\lambda _i }}
- 1} } \right)(s) = \sum\limits_{i = 1}^{n + 1} {\frac{{\mu _i
\lambda' _i }}{{\lambda _i }}} (y_i )
\end{equation*}
and
\begin{equation*}
(s) = \frac{1}{{\rho _s }}\mathscr{T} \sum\limits_{i = 1}^{n + 1}
{\mu _i (x_i)}  = \sum\limits_{i = 1}^{n + 1} {\frac{{\rho _i \mu _i
}}{{\rho _s }}} (y_i )
\end{equation*}

Since the linear expression of $(s)$ by $(y_i)$ is unique once the homogeneous coordinate vectors are fixed, comparing the corresponding factors before each $(y_i)$, we can obtain:
\begin{equation}\label{eigenvalueconstant}
\frac{{\rho _i \lambda _i }}{{\lambda' _i }} =\text{ const.} \qquad
(\forall\;i\,=\,1,\cdots,n+1)
\end{equation}

According to equation \eqref{s-rho}
\begin{equation}(s)=\frac{1} {{\rho _s
}}\mathscr{T}\cdot (s) ,
\end{equation}
which yields:
\begin{align}\label{s-for-polynomial}
\nonumber (s) &= \lambda _i (x_i) + \lambda' _i (y_i ) = \lambda _i (x_i) +\frac{{\lambda' _i }} {{\rho _i
}}\mathscr{T}\cdot (x_i)\\
& = \frac{{\lambda_i }} {{\rho _s }}\mathscr{T}\cdot (x_i) + \frac{{\lambda' _i }} {{\rho_s \rho _i }}\mathscr{T} ^2\cdot (x_i) \qquad \forall \: i = 1, \cdots,   n+1 \qquad \qquad \raisetag{50pt}
\end{align}

then from:
 \[
\lambda _i (x_i) +\frac{{\lambda' _i }} {{\rho _i }}\mathscr{T}
(x_i) =\frac{{\lambda_i }} {{\rho _i }}\mathscr{T} (x_i) +
\frac{{\lambda' _i }} {{\rho _s \rho _i }}\mathscr{T} ^2 (x_i)
\]
we have:
\begin{align}
(\mathscr{T}  - \rho _s \boldsymbol{I})\left( {\mathscr{T}  + \frac{{\rho _i \lambda _i }} {{\lambda '_i }}\boldsymbol{I}} \right)(x_i) = \boldsymbol{0}\\
\nonumber \forall \: i = 1,\: \cdots, n+1
\end{align}

 Since $\rho_s$ is an eigenvalue of the transformation matrix $\mathscr{T}$, and $\mathscr{T}$ $\neq$ $\boldsymbol{I}$,
\begin{equation}(t - \rho _s )\left( {t + \frac{{\rho _i \lambda _i }}
{{\lambda '_i }}} \right)
\end{equation}
 is the minimal characteristic polynomial of the transformation matrix $\mathscr{T}$. Consequently,
\begin{equation}\label{eigenvalue}
 -\frac{{\rho _i \lambda _i }}{{\lambda '_i }} = {\rm const.}
 \end{equation} is also an eigenvalue  of transformation matrix $\mathscr{T},$ and then it can be proved that the geometric multiplicity thereof is $n$.

Then prove that all the homogeneous coordinate vectors of the $C_{n+1}^2$ different intersection points, denoted as ($s)_{i,j}$, ($i$\,$\neq$\,$j$,\:$i$,$j$=1,$\cdots$,$n$+1), are the associated eigenvectors with the eigenvalue {equation} \eqref{eigenvalue}.

Because:
\qquad $$(s)=\lambda_i (x_i)+\lambda'_i (y_i) \quad (i {\text = 1,} \cdots, n{\text +1} ), $$

\begin{align}\label{intersections-1}
\nonumber (s)_{i,j}&=\lambda_i(x_i)-\lambda_j(x)_j = \lambda'_j(y)_j-\lambda'_i (y_i) \\
        & \forall\; i \neq j, \quad i, j = 1, \cdots, n{\text +1}
\end{align}
and
\begin{align}\label{intersections-t}
\mathscr{T}\cdot (s)_{i,j}&=\mathscr{T}\cdot(\lambda_i(x_i)-\lambda_j(x)_j) = \lambda_i \rho_i (y_i)-\lambda_j \rho_j (y)_j
\end{align}

Comparing the linear expression of $(s)_{i,j}$ in equation \eqref{intersections-1} and $\mathscr{T}\cdot (s)_{i,j}$ in equation \eqref{intersections-t} by $(y)_j$ and $(y_i)$, considering {equation} \eqref{eigenvalueconstant}, we obtain:

\begin{align}\label{intersections-eigenvalue}
\nonumber \mathscr{T}\cdot (s)_{i,j}& = -\frac{{\rho _i \lambda _i }}{{\lambda' _i }} (s)_{i,j}\\
 & \forall\; i \neq j, \quad i, j = 1, \cdots, n{\text +1}
\end{align}
Since the rank of the vector set which consists of all the homogeneous coordinates of the $C_{n+1}^2$ intersection points is $n$, and all the vectors are associated eigenvectors, the geometric multiplicity  of the associated eigenvalue is $n$.

When $\mathscr{T}$ is a singular matrix, the proof is similar with only slight difference.
\end{proof}


{\footnotesize
\begin{thebibliography}{99}

\bibitem {investigation}
{ Chen Ziqiang (2000).
\newblock An investiation of perspective projection.
\newblock {\em Journal of East China University of Science \& Technology} (in Chinese), 26(2):201--205.}

\bibitem {meaning}
{ Chen Ziqiang (2005).
\newblock Meaning of elementary matrices in projective geometry and its applications.
\newblock {\em Progress in Natural Science } (in Chinese), 15(9):1113--1122.}

\bibitem{Faugeras}
{ Faugeras Olivier, Luong Quang-Tuan, and Papadopoulo Th\'eo (2001).
\newblock {\em The Geometry of Multiple Images,---The Laws That Govern the Formation of Multiple
Images of a Scene and Some of Their Applications}.
\newblock The MIT Press, Cambridge, Massachusetts, London, England.}

\bibitem{Goldman}
{ Goldman Ronald (2009).
\newblock {\em An Integrated Introduction to Computer Graphics and Geometric Modeling}. 1st Edition.
\newblock CRC Press, Inc., Boca Raton, FL, USA.}

\bibitem{Hartley}
{ Hartley Richard I., Zisserman Andrew (2004).
\newblock {\em Multiple View Geometry in Computer Vision}. Second Edition.
\newblock Cambridge University Press, New York.}

\bibitem{Horn}
{ Horn Roger A. (1990).
\newblock {\em Matrix Analysis}.
\newblock {Cambridge University Press, New York.}}

\bibitem{Householder1958-1}
{ Householder Alston. S. (1958).
\newblock {\em A class of methods for inverting matrices}.
\newblock {\em Journal of the Society for Industrial and Applied Mathematics}, vol.6(2): 189--195.}

\bibitem{Householder1958-2}
{ Householder Alston. S. (1958).
\newblock {\em Unitary triangularization of a nonsymmetric matrix}.
\newblock {\em Journal of the Association for Computing Machinery}, vol.5(4): 339--342.}

\bibitem{Householder}
{ Householder Alston. S. (1964).
\newblock {\em The Theory of Matrices in Numerical Analysis}.
\newblock Blaisdell Publishing Company., New York.}

\bibitem{Kline}
{ Kline Morris (1972).
\newblock {\em Mathematical Thought from Ancient to Modern Times}.Vol.3.
\newblock {Oxford University Press, New York.}}

\bibitem{CAD}
{ Marsh Duncan (2005).
\newblock {\em Applied Geometry for Computer Graphics and CAD}. 2nd Ed.
\newblock Springer-Verlag London Limited.}

\bibitem{Mathews}
{ Mathews G. B. (1914).
\newblock {\em Projective Geometry}.
\newblock Longmans, Green and Co., London.}

\bibitem{Perspectives}
{ Richter-Gebert J\"{u}rgen (2011).
\newblock {\em Perspectives on Projective Geometry}.
\newblock Springer-Verlag, Berlin, Heidelberg.}

\bibitem{Salomon}
{ Salomon David (2006).
\newblock {\em Transformations and Projections in Computer Graphics}.
\newblock Springer-Verlag London Limited.}

\bibitem{planeTransform}
{ Semple J. G., Kneebone G. T. (1952).
\newblock {\em Algebraic Projective Geometry}. Reprinted in 1963.
\newblock Oxford University Press, Ely House, London.}

\bibitem{Ueberberg}
{ Ueberberg Johannes (2011).
\newblock {\em Foundations of Incidence Geometry --- Projective and Polar Spaces}.
\newblock Springer-Verlag, Berlin, Heidelberg.}

\bibitem{Vaisman}
{ Vaisman Izu (1997).
\newblock {\em Analytical Geometry}.
\newblock World Scientific Publishing Co. Pte. Ltd., Singapore.}

\bibitem{Veblen}
{ Veblen Oswald, and Young John Wesley (1910).
\newblock {\em Projective Geometry}. Vol.1
\newblock Blaisdell Publishing Company, New York, Toronto, London.}

\bibitem{PointTransform}
{ Verdina Joseph (1971).
\newblock {\em Projective Geometry and Point Transformations}.
\newblock Allyn and Bacon, Inc., Boston.}

\bibitem{VinceFormulae}
{ Vince John A. (2005).
\newblock {\em Geometry for computer graphics : formulae, examples and proofs}.
\newblock Springer-Verlag London Limited.}

\bibitem{VinceRotation}
{ Vince John A. (2005).
\newblock {\em Rotation Transforms for Computer Graphics}.
\newblock Springer-Verlag London Limited.}

\end{thebibliography}}
\end{document}